\documentclass[sigconf, review=False, nonacm, balance=True]{acmart}

\AtBeginDocument{%
  \providecommand\BibTeX{{%
    \normalfont B\kern-0.5em{\scshape i\kern-0.25em b}\kern-0.8em\TeX}}}

\usepackage{comment}
\usepackage{asymptote}
\usepackage{algorithm}
\usepackage{enumerate}
\usepackage[noend]{algpseudocode}
\usepackage{bm}

\newtheorem{lemma}{Lemma}
\newtheorem{theorem}{Theorem}
\newtheorem{definition}{Definition}

\usepackage{xcolor}
\newcommand{\norm}[1]{\left\lVert#1\right\rVert}

\newcommand{\yiheng}[1]{\textcolor{purple}{[Yiheng says: #1]}}

\newcommand{\addcites}[1]{\textcolor{red}{[Need citations!]}}

\DeclareMathOperator*{\argmin}{arg\,min}

\allowdisplaybreaks

\begin{document}

\title{Online Optimization with Predictions and Non-convex Losses}

\author{Yiheng Lin}
\authornote{This work was done when Yiheng Lin was visiting California Institute of Technology.}
\authornote{This work was supported by NSF grants  AitF-1637598 and CNS-1518941, with additional support for Gautam Goel provided by an Amazon AWS AI Fellowship.}
\affiliation{%
  \institution{IIIS, Tsinghua University}
  \city{Beijing}
  \country{China}
}

\author{Gautam Goel}
\authornotemark[2]
\affiliation{%
  \institution{California Institute of Technology}
  \city{Pasadena}
  \state{California}
  \country{USA}
}

\author{Adam Wierman}
\authornotemark[2]
\affiliation{%
  \institution{California Institute of Technology}
  \city{Pasadena}
  \state{California}
  \country{USA}
}

\begin{abstract}



We study online optimization in a setting where an online learner seeks to optimize a per-round hitting cost, which may be non-convex, while incurring a movement cost when changing actions between rounds. We ask: \textit{under what general conditions is it possible for an online learner to leverage predictions of future cost functions in order to achieve near-optimal costs?} Prior work has provided near-optimal online algorithms for specific combinations of assumptions about hitting and switching costs, but no general results are known.  In this work, we give two general sufficient conditions that specify a relationship between the hitting and movement costs which guarantees that a new algorithm, Synchronized Fixed Horizon Control (SFHC), achieves a $1+O(1/w)$ competitive ratio, where $w$ is the number of predictions available to the learner.  Our conditions do not require the cost functions to be convex, and we also derive competitive ratio results for non-convex hitting and movement costs. Our results provide the first constant, dimension-free competitive ratio for online non-convex optimization with movement costs. We also give an example of a natural problem, Convex Body Chasing (CBC), where the sufficient conditions are not satisfied and prove that no online algorithm can have a competitive ratio that converges to 1. 


\end{abstract}

\maketitle

\section{Introduction}

Online optimization is a classical area in online learning with a long and impactful history.  In this paper, we study a variation of online optimization where the learner incurs a movement (switching) cost associated with the change in actions between consecutive rounds.  Specifically, we study online optimization in a setting where an online learner interacts with the environment in a sequence of rounds $1 \ldots T$. In each round, a cost function $f_t : \mathbb{R}^d \rightarrow \mathbb{R}_{\geq 0}$ is revealed and the learner chooses a point $x_t \in \mathbb{R}^d$ in response. After picking its point, the learner pays a \textit{hitting cost} $f_t(x_t)$ as well as a \textit{movement (switching) cost} $c(x_t, x_{t-1})$, which penalizes the learner for changing its actions between rounds. The movement cost adds a considerable degree of complexity to the decision making of the learner since it couples the learner's actions between rounds. A choice $x_t$ which is near the minimizer of $f_t$ and incurs little cost in round $t$ might turn out to be far away from the minimizer of $f_{t+1}$.  Thus, the learner must balance decisions about $x_t$ with the potential of movement costs in the future.  However, the learner does not have information about future hitting costs, which makes it difficult to choose the correct balance.   

The most prominent version of online optimization with movement costs is known as \textit{Smoothed Online Convex Optimization (SOCO)}, and assumes that the hitting costs are convex and the movement costs are a norm. SOCO has attracted considerable attention in the past decade, e.g., \cite{li2018online, chen2018smoothed, goel2018smoothed, shi2019value, goel2019beyond, lin2012online, lin2013dynamic, chen2015online, chen2016using, badiei2015online},  driven in part by its connection to classical online algorithms problems such as Convex Body Chasing (CBC) \cite{bubeck2018chasing, argue2019nearly, sellke2019chasing, argue2019chasing}, Metrical Task Systems (MTS) \cite{borodin1992optimal, bartal1997polylog, blum2000line}, and the $k$-server problem \cite{manasse1990competitive, bubeck2018k, buchbinder2019k}.  Additionally, much of the work on SOCO has been driven by its many applications, e.g., speech animation \cite{kim2015decision}, control \cite{goel2017thinking, goel2018smoothed}, smart grid \cite{kim2014real}, video streaming \cite{joseph2012jointly}, and data centers \cite{lin2012online, comden2019online}.
 
While initial results on SOCO provided algorithms with performance guarantees in only limited settings, e.g., \cite{lin2012online} provides a 2-competitive algorithm for 1-dimensional SOCO problems, at this point algorithms that have constant dimension-free competitive ratios in high-dimensional settings have been discovered, e.g., \cite{goel2018smoothed, goel2019beyond} provide a constant-competitive algorithm for strongly convex hitting costs and squared $\ell_2$ movement costs.  However, while it is possible to provide dimension-free, constant competitive algorithms in settings where the learner has no information about future hitting cost functions, in most applications where SOCO is used it is possible to make accurate predictions of future costs.  Such predictions are extremely valuable for the online learner and, as a result, a growing literature has considered situations where the learner has access to predictions of future costs, e.g., \cite{lin2012online, chen2015online, chen2016using, badiei2015online, li2018online, comden2019online, shi2019value}.  Most typically, this stream of work considers that the learner has access to perfect predictions of the next $w$ costs, but in some cases it is possible to extend such results to noisy predictions as well, e.g., \cite{chen2015online, chen2016using}.  

Clearly, the use of predictions is beneficial for the learner.  With access to $w$ perfect predictions, it is possible for the learner to obtain a competitive ratio that converges to 1 as $w\to\infty$.  The first such result to appear was \cite{lin2012online}, which provides an algorithm that has a competitive ratio of $1+O(1/w)$ when hitting costs are the operating cost for servers and movement costs are incurred by toggling into and out of a power-saving mode between timeslots.  Since then, other results have followed, e.g., \cite{chen2015online, chen2016using, badiei2015online, li2018online, shi2019value}, and when more stringent requirements on the cost functions are considered it is possible to design algorithms whose competitive ratio converges to 1 exponentially quickly in $w$ \cite{li2018online}.  

The discussion above highlights that there has been considerable progress in the design of competitive algorithms for SOCO, both with and without access to predictions.  However, at this point all existing results require specific assumptions on both the hitting costs and movement costs, e.g., when the hitting costs are $\alpha-$polyhedral and convex,  while the movement cost is given by $\norm{x_t - x_{t-1}}_2$ (see \cite{chen2018smoothed}); when the hitting costs are $m-$strongly convex, while the movement cost is given by $\frac{1}{2}\norm{x_t - x_{t-1}}_2^2$ (see \cite{goel2018smoothed}).  In this paper, instead of studying a specific class of costs, we ask: \textit{under what general conditions is it possible for an online learner to achieve near-optimal costs both with and without predictions?} In particular, is it possible to obtain constant-competitive algorithms without assumptions like strong-convexity and local polyhedrality; potentially even in the case of \textit{non-convex costs}?

The case of non-convex costs is particularly tantalizing given the importance of non-convex losses for machine learning and the prominence of non-convex costs in applications such as power systems and networking. Techniques from non-convex optimization have been applied to a wide variety of problems in machine learning, including matrix factorization, phase retrieval, and sparse recovery; we refer the interested reader to \cite{jain2017non} for a recent survey. The Optimal Power Flow (OPF) problem at the core of the operation of power systems is also non-convex \cite{low2014convex, low2014convex2}; thus requiring online non-convex optimization for real-time control. Non-convex optimization in online settings has also been studied in a variety of other contexts, such as portfolio optimization \cite{ardia2010differential, krokhmal2002portfolio} and support vector machines \cite{ertekin2010nonconvex, mason2000improved}, among many others.

\subsection*{Contributions of this paper} 

In this paper we introduce two general, sufficient conditions (see Section \ref{sec:pred_helps}) under which is possible to achieve a constant competitive ratio without predictions \textit{and} to leverage predictions to achieve near-optimal cost, i.e., a $1+O(1/w)$ competitive ratio. Importantly, these conditions do not require convexity of the hitting or movement costs.  

The first sufficient condition is an order of growth condition that ensures the hitting cost functions grow at least as quickly as the switching costs as one moves away from the minimizer.  The second condition requires that the switching costs satisfy an approximate version of the triangle inequality. Nearly all assumptions made in previous papers on online optimization with movement costs are special cases of these conditions, e.g., locally polyhedral costs \cite{chen2018smoothed}, strongly convex costs \cite{goel2018smoothed, goel2019beyond}, and more \cite{lin2012online, liu2011geographical}.  While we do not prove that these conditions are necessary, we show in Section \ref{sec:limit_of_pred} that an important class of online optimization problems, namely Convex Body Chasing, violates the conditions, and that furthermore it is impossible for an online learner to leverage predictions to achieve near-optimal costs in this class. We note that the Convex Body Chasing problem has attracted much recent attention (see \cite{argue2019nearly, sellke2019chasing, argue2019chasing}).

To show that these two conditions are sufficient, we propose a novel algorithm, Synchronized Fixed Horizon Control (SFHC), and show that it is constant-competitive whenever the two conditions hold, including both when the cost functions are convex and non-convex.  More specifically, we introduce two variants of SFHC, Deterministic SFHC and Randomized SFHC. 

In the case when costs are convex, Deterministic SFHC provides a competitive ratio of $\max\Big(1 + \frac{\eta + \eta^2}{2\lambda}, \eta^2\Big)$ without access to predictions and a competitive ratio of $1+O(1/w)$ in the case of predictions (Theorem \ref{thm:ARFHC_general_cr_1}). Thus, SFHC unifies two distinct lines of inquiry in the literature: how to design algorithms take advantage of predictions when they are available \cite{lin2012online, chen2015online, chen2016using, li2018online, shi2019value} and how to design algorithms that work when predictions are not available \cite{goel2018smoothed, chen2018smoothed,  bansal20152, lin2013dynamic, goel2019beyond}.  \textit{SFHC is the first algorithm to provide a constant-competitive guarantee in both settings.}  

In the case when costs are non-convex, Deterministic SFHC maintains a competitive ratio of $\max\left(1 + \frac{\eta + \eta^2}{2\lambda}, \eta^2\right)$ without access to predictions but provides a competitive ratio of $C+O(1/w)$ in the case of predictions, where $C>1$ (Theorem \ref{thm:AFHC_non_convex_1}). Thus, it does not leverage predictions to ensure near-optimal cost. However, randomization can be used to improve the result in the case of predictions.  Specifically, Randomized SFHC provides a competitive ratio of $1+O(1/w)$ for general non-convex functions that satisfy our sufficient conditions, given an oblivious adversary (Theorem \ref{thm:nonconvex_rsfhc}).  Further, the result extends (with slight modifications to the design of Randomized SFHC) to the case of a semi-adaptive adversary (Theorem \ref{thm:RRFHC_general_cr_1}). These results represent \textit{the first constant-competitive guarantees for online optimization with movement costs and non-convex losses.}

The design of SFHC is inspired by the design of Averaging Fixed Horizon Control (AFHC) \cite{lin2012online}, which has served as the basis for many algorithms in this space, e.g., \cite{chen2016using, shi2019value}.  Like AFHC, Deterministic SFHC works by averaging the choices of $w$ different subroutines. However, the subroutines are very different than AFHC. 
At each time step $\tau$, one of the subroutines of AFHC optimizes the cost over the window $[\tau, \tau + w -1]$ given the starting state $x_{\tau-1}$ (see \cite{lin2012online}). The SFHC subroutines perform a similar optimization, but with an additional constraint that the point selected at the end of the window is ``synced" to the minimizer of the hitting cost at that timestep. These synchronization points ensure that, when the sufficient conditions hold, the algorithm does not drift too far from the actions of the offline optimal.  Thus, rather than optimize cost, SFHC is designed to track the offline optimal (which also implicitly leads to achieving good cost).  The key difference between Deterministic SFHC and Randomized SFHC is that Randomized SFHC chooses an action of a subroutine uniformly at random rather than averaging the choices of the subroutines.  It is perhaps surprising that randomization helps in the case of non-convex costs given that \cite{bansal20152} shows that randomization cannot help in the case of SOCO.  


\subsection*{Related literature} 
There is a large literature on online optimization, both with and without switching costs. In the setting without switching costs, most work has focused on Online Convex Optimization (OCO) \cite{hazan2016introduction}.  This problem is similar to SOCO, except that (i) there are no switching costs and (ii) the online learner picks the point $x_t$ before observing the cost $f_t$. In this problem, the goal is to design algorithms with low \textit{regret}, i.e., the goal is to find a strategy that tracks the cost of the best fixed action as closely as possible. In 2003, Zinkevich described Online Gradient Descent, the first algorithm to achieve sublinear regret for OCO \cite{zinkevich2003online}. This was subsequently generalized by algorithms such as Online Mirror Descent \cite{nemirovsky1983problem, bansal2017potential} and the Multiplicative Weights Update algorithm \cite{arora2012multiplicative} . 

Beyond the case of convex costs, online non-convex optimization (without switching costs) has also received considerable attention, e.g.,  \cite{yang2018optimal, ardia2010differential, krokhmal2002portfolio, ertekin2010nonconvex, mason2000improved}.  Most commonly the algorithms used in these papers are variations of Online Exponential Weights.  For example, recently, \cite{yang2018optimal} presents an algorithm, Online Recursive Weighting, that achieves a bound on regret that matches the lower bound in the convex setting.  

All papers above consider problems without movement costs. The inclusion of movement costs makes the problem considerably more challenging and motivates the use of a different performance measure.  Specifically, instead of regret, algorithms are evaluated with respect to the competitive ratio.  In fact, it is known that there is a fundamental incompatibility between regret and competitive ratio: it is impossible to create an algorithm for SOCO with both sublinear regret and constant competitive ratio \cite{andrew2013tale}.

In the case of movement costs, all previous papers focus on the convex setting.  In particular, the problem of Smoothed Online Convex Optimization (SOCO), i.e. OCO with switching costs, was introduced in \cite{lin2013dynamic} in the context of dynamic power management in data centers. Since then it has been applied across many domains, including speech animation \cite{kim2015decision}, multi-timescale control \cite{goel2017thinking}, video streaming \cite{joseph2012jointly}, thermal management of System-on-Chip (SoC) circuits \cite{zanini2010online}, and power generation planning \cite{kim2015decision}.

The original paper introducing SOCO \cite{lin2013dynamic} gave a 3-competitive algorithm for one dimensional action spaces. Following this work, an algorithm with a competitive ratio of 2 was introduced in \cite{bansal20152} and this was shown to be optimal in \cite{antoniadis2017tight}.
Until recently, there were no algorithms for SOCO that worked beyond one dimension ($d = 1$) that did not use predictions.  However, last year it was shown that it is possible to design competitive algorithms for SOCO beyond one dimension, provided the hitting cost functions have some structure.  Specifically, in \cite{chen2018smoothed}, Online Balanced Descent (OBD) was introduced and shown to have a dimension-free competitive ratio in the special case where the cost functions are polyhedral. Following this work, it was shown that OBD also provides a dimension-free competitive ratio when the hitting costs are strongly convex \cite{goel2018smoothed} and that a variant of OBD called Regularized OBD achieves the optimal competitive ratio when hitting costs are strongly convex. We note that this literature is fairly distinct from the work involving predictions. Up until now, the algorithms that are designed to be competitive without predictions are not able to take advantage of predictions when they are available.

In many applications the online learner has some information about future costs, and making use of these predictions of future costs is crucial.  This has prompted a great deal of work involving the design of algorithms that leverage predictions, e.g., \cite{li2018using, chen2016using, badieionline, chen2015online, shi2019value, comden2019online}. Most of this work considers models where the online learner has a prediction window of length $w$, i.e. at time $t$, the agent observes the cost functions $f_t \ldots f_{t+w-1}$ before choosing the point $x_t$ (the case $w = 1$ captures the standard SOCO setting). Naturally, as $w$ tends to infinity the algorithm has more and more information and hence should achieve better performance. In \cite{chen2015online}, it was shown that, surprisingly, Receding Horizon Control (RHC) cannot guarantee a competitive ratio that converges to one as $w$ tends to infinity; however, it was also shown that Averaging Fixed Horizon Control (AFHC) can guarantee a near-optimal competitive ratio if $w$ is sufficiently large. Later, it was shown that it is possible to obtain algorithms whose competitive ratio decays exponentially in $w$ in the setting where the hitting costs are both strongly convex with bounded gradients and uniformly bounded below by a constant and the movement cost is quadratic \cite{li2018using}. However, it is again important to note that this literature is fairly distinct from the work of designing algorithms that are competitive without predictions.  Up until now, the algorithms that are designed to be competitive with predictions are not able to be competitive without the use of predictions. 

While there has been considerable progress on designing algorithms for SOCO both with and without predictions, to this point the results all rely on specific structural assumptions about the costs, and no previous work extends to non-convex costs. In this paper, we do not directly make strong structural assumptions about the hitting functions or the switching costs; instead, we ask \textit{under what general conditions on the hitting costs and switching costs is it possible to design competitive algorithms?} Surprisingly, we show that convexity is not a necessary condition to design a competitive algorithm; this allows us to tackle a much broader range of hitting and movement costs than prior work.

\section{Problem Formulation}\label{sec:preliminary}
In this paper we study the problem of online (non-convex) optimization with switching costs. An instance of this problem consists of an initial point $x_0 \in \mathbb{R}^d$, a sequence of \textit{hitting cost} functions $f_1 \ldots f_T : \mathbb{R}^d \rightarrow \mathbb{R}_{\geq 0}$, and a \textit{switching cost}, a.k.a., \textit{movement cost}, $c: \mathbb{R}^d \times \mathbb{R}^d \rightarrow \mathbb{R}_{\geq 0}$. The sequence of hitting costs is incrementally revealed to an online learner, who picks points in response to observing the hitting cost function and incurs costs associated with the choice. 

More precisely, in the $t$-th round, the function $f_t$ is revealed to the online learner, who picks a point $x_t$ in response, and incurs the cost $f_t(x_t) + c(x_t, x_{t-1})$. The $c(x_t, x_{t-1})$ term acts as a regularizer, discouraging the learner from changing its action between rounds. Note that, without the switching cost, it is easy for the online learner to incur the optimal cost: it simply picks $x_t = \argmin_x f_t(x)$ in every round.  Thus, it is the presence of the switching cost which makes the problem interesting and non-trivial, as it couples costs between rounds.

The online learner seeks to minimize its cumulative cost across all rounds:
\begin{equation}\label{eq:total_cost}
    cost(ALG) = C_T(x) := \sum_{t=1}^T f_t(x_t) + c(x_t, x_{t-1}), 
\end{equation}
where $C_T: \mathbb{R}^{d\times T} \to \mathbb{R}^+\cup \{0\}$ is a function that computes the total cost incurred by a sequence $x$.
We measure the performance of the online learner by comparing its cost to the offline optimal cost, i.e. the cost given full knowledge of the functions $f_t$:
$$cost(OPT) = \min_{x_1^*, \ldots x_T^*} \sum_{t=1}^T f_t(x_t^*) + c(x_t^*, x_{t-1}^*). $$

The goal of this paper is to design strategies for the online learner so that it incurs nearly the same cost as the offline optimal.  This can be measured either in terms of \emph{regret}, which compares to the offline static optimal, or the \textit{competitive ratio}, which compares to the offline dynamic optimal.  Our focus in this paper is on the competitive ratio, which is a more challenging measure to achieve good performance under.  Formally, the competitive ratio is defined as $\sup_{f_1, f_2, \cdots, f_T} cost(ALG)/cost(OPT)$ and, if the online learner can pick a strategy which guarantees that the competitive ratio is finite for any sequence of hitting costs, we say that the online learner's strategy is \textit{competitive}.

In this paper, we seek to derive an algorithm that can be competitive in settings where hitting costs are non-convex.  This paper is the first to consider costs that are non-convex in the context where movement costs are present. (Non-convex online optimization without movement costs has been considered in, e.g., \cite{yang2018optimal, ardia2010differential, krokhmal2002portfolio, ertekin2010nonconvex, mason2000improved}).  In all prior work that considers movement costs,  hitting costs are assumed to be convex and so the problem is typically referred to as \textit{Smoothed Online Convex Optimization (SOCO)}. This problem was first introduced in the context of dynamic power management in data centers in \cite{lin2013dynamic}.

In many applications, the classical formulation of an online learner is too restrictive.  It is not true that the learner has no information about future costs, instead the learner has the ability to derive (noisy) forecasts of future cost functions.  As a result, there has been a great deal of work focused on designing algorithms for online learners that have access to predictions of future costs \cite{lin2012online, chen2015online, chen2016using, li2018online, shi2019value}.  This line of work, initiated by \cite{lin2012online}, seeks to design algorithms which have competitive ratios that converge to $1$ as the number of predictions available to the algorithm, $w$, grows.  More specifically, in this line of work, at time $t$ an online learner with prediction window $w$ observes the cost functions  $f_t \ldots f_{t+w-1}$ before choosing the point $x_t$.  Note that the case of $w = 1$ captures the standard SOCO setting.  Given these predictions, the learner seeks to have a competitive ratio of the form $1 + g(w)$, where $g(w)\to 0$ as $w\to\infty$.  Thus, as the number of predictions grows the cost of the learner converges to the offline optimal cost.  Under well-behaved costs it is sometimes possible for $g(w)$ to decay exponentially, e.g. \cite{li2018online}; however for general cost functions polynomial decay is the goal, e.g., \cite{chen2016using}.

Notice that, in the formulation just described, the predictions of future cost functions given to the learner are \emph{perfect}. While in real applications predictions are noisy, due to technical challenges this assumption is common, e.g., see \cite{lin2012online, li2018online}. There does exist some work that has extended the results to noisy predictions in limited cases, e.g., \cite{chen2015online, chen2016using}; however, in general, the extensions to noisy predictions are difficult and, when possible, have confirmed the insights initially proven in models with perfect predictions.  In this work, we focus on the perfect prediction model.  Given the challenges associated with providing results for non-convex costs, this is natural and necessary.  However, we do intend to investigate extensions of these results to noisy predictions in future work.  
\section{The limited power of predictions}\label{sec:limit_of_pred}
To this point, the literature studying online optimization with predictions has focused on positive results, i.e., providing algorithms that can achieve competitive ratios which converge to $1$ as the number of predictions available to the algorithm grows, e.g., \cite{lin2012online, chen2016using, li2018online, shi2019value}.  However, all positive results that exist apply to only specific forms of hitting and switching costs.  As a result, an important question that remains for the community is: \textit{Is it always possible for an online learner to leverage predictions to achieve a near-optimal competitive ratio?} 

In this section, we show that the answer is ``no.''  We show that there exist instances where predictions cannot guarantee the learner a near-optimal cost, even in the case when cost functions are convex.  Further, the instances we construct are not strange corner cases, they include an important subclass of online optimization, Convex Body Chasing (CBC), which has received considerable attention in recent years, e.g., \cite{argue2019nearly, sellke2019chasing, argue2019chasing}. 

In the following, we detail a construction that highlights a fundamental challenge for online optimization with predictions.  In particular, all previous positive results rely on a condition such as strong convexity or local polyhedrality, which rules out convex hitting cost functions that are ``flat'' around the minimizer. Our construction shows that, if this condition is not satisfied, it is not guaranteed that predictions can be leveraged by the online learner.

\subsection*{Convex Body Chasing}  As in Online Convex Optimization (OCO),  an instance of CBC consists of an online agent making decisions in a series of rounds.  In each round, the agent is presented a closed convex set $K_t \subseteq \mathbb{R}^d$. After observing the convex body, the  agent picks a point $x_t \in K_t$ and pays the movement cost $c(x_t, x_{t-1}) = \norm{x_t - x_{t-1}}$, where $\norm{\cdot}$ is a norm. Thus, the total cost incurred by the online agent is:
$$cost(ALG) = \sum_{t=1}^T c(x_t, x_{t-1}).$$

It is straightforward to see that CBC is a special case of SOCO where the hitting cost functions $f_t$ are indicator functions of the bodies $K_t$.  It is the form of this hitting cost function that creates a challenge for the learner. The fact that hitting cost functions are flat within the body means that the offline optimal can be anywhere in the body without paying an additional cost relative to that incurred by the learner. This makes it much more difficult for the learner to match the cost of the offline optimal; intuitively the learner has no information about where in the body the offline point is located, making it harder to track the offline points.

While it is easy to see that CBC is an instance of SOCO, it is perhaps surprising to discover that a SOCO problem can also be viewed as an instance of CBC. Specifically, if there exists an algorithm that can solve the CBC problem in $(d+1)-$dimensional space, we can construct an algorithm that can solve the SOCO problem in $d-$dimensional space. This fact was first noted in \cite{antoniadis2016chasing}, but the result was retracted.  Then, it was noted without a formal proof in \cite{bubeck2019improved}.  Here we provide a formal statement and proof of the reduction since the result is crucial to our goal of highlighting limitations on the power of predictions in online convex optimization. 

\begin{proposition}\label{prop:Reduce_SOCO_to_CBC}
Consider a $d-$dimensional SOCO problem where the movement cost function is given by $c(x_t, x_{t-1}) = \norm{x_t - x_{t-1}}_p$. Suppose Algorithm $A$ is $C$-competitive algorithm for CBC in $d + 1$ dimensions with movement cost function $c(x_t, x_{t-1})$.  Then, there exists a $4C$-competitive algorithm $A'$ for the $d-$dimensional SOCO problem.
\end{proposition}

Given Proposition \ref{prop:Reduce_SOCO_to_CBC}, we can obtain insight on the limitations of the power of predictions in SOCO by studying the power of predictions in CBC.  The following theorem shows that it is not possible to use predictions to obtain a near-optimal competitive ratio in CBC.

\begin{theorem}\label{thm:Pred_not_help}
Consider an instance of CBC in $d$-dimensional space with movement cost function $c(x_t, x_{t-1}) = \norm{x_t - x_{t-1}}$, where $\norm{\cdot}$ is an arbitrary norm. Suppose $g(d)>1$ is a lower bound on the competitive ratio for all algorithms when the length of the prediction window is $1$. For any $w>0$, the same lower bound holds for any algorithm with a prediction window of length $w$.
\end{theorem}

To interpret this theorem, we make note of an important lower bound in the SOCO literature.  In particular, any online algorithm that can only see the current cost function (i.e., has $w=1$) has competitive ratio $\Omega(\sqrt{d})$ when the movement costs are given by the standard $\ell_2$ norm, \cite{chen2018smoothed}.  Thus, Theorem \ref{thm:Pred_not_help} implies that any online algorithm for CBC with a finite prediction window has competitive ratio lower bounded by $\Omega(\sqrt{d})$.

\begin{proof}[Proof of Theorem \ref{thm:Pred_not_help}]
Suppose an algorithm A can leverage a prediction window of length $w$ and achieve a competitive ratio of $f(d, w)$. It suffices to give an algorithm $A'$ which only needs a prediction window of length $1$ but achieves the same competitive ratio with $A$.

In the proof, we construct algorithm $A'$ using $A$ as an oracle. When a convex body $K_t$ arrives at timestep $t$, we duplicate it $w$ times and feed these $w$ convex bodies in a sequence to algorithm $A$. Specifically, at timestep $t$, we construct $w$ convex bodies $K_{t, 1} = K_{t, 2} = \cdots = K_{t, w} = K_t$ and feed the sequence $\{K_{t, 1}, K_{t, 2}, \cdots, K_{t, w}\}$ to $A$. $A$ is provided with a prediction window with length $w$ and is required to chase $wT$ convex bodies in the order:
$$K_{1, 1}, \cdots, K_{1, w}, K_{2, 1}, \cdots, K_{2, w}, \cdots, K_{T, 1}, \cdots, K_{T, w}.$$
We call this duplicated convex body chasing game with prediction $\mathcal{I}'$ to distinguish it with the original game $\mathcal{I}$ in which the learner only needs to chase $T$ convex bodies in the order $K_1, \cdots, K_T$. For convenience, we use $cost(ALG, I)$ to denote the total cost incurred by algorithm $ALG$ in instance $I$. We use $cost(OPT, I)$ to denote the offline optimal cost in instance $I$. 

As a result of the competitive ratio guarantee for $A$, we have that
\begin{equation}\label{thm:Pred_not_help:e1}
    cost(A, \mathcal{I}') \leq f(d, w)\cdot cost(OPT, \mathcal{I}').
\end{equation}
Suppose the sequence of points picked by $A$ in instance $\mathcal{I}'$ is
$$x_{1, 1}, \cdots, x_{1, w}, x_{2, 1}, \cdots, x_{2, w}, \cdots, x_{T, 1}, \cdots, x_{T, w}.$$
We instruct $A'$ to pick $x_i = x_{i, 1}, i = 1, 2, \cdots, T$ in instance $\mathcal{I}$. Notice that $A$ only looks at $K_{i, 1}, \cdots, K_{i, w}$ when it picks $x_{i, 1}$ in instance $\mathcal{I}'$. Since the bodies $K_{i, 1}, \cdots, K_{i, w}$ are just duplicates of body $K_i$, $A'$ is an online algorithm with prediction window $w = 1$ in game $\mathcal{I}$. It follows that
\begin{equation}\label{thm:Pred_not_help:e2}
    cost(A', \mathcal{I}) \leq cost(A, \mathcal{I}'),
\end{equation}
because, by the triangle inequality, we have that
\begin{equation*}
    \begin{aligned}
    \norm{x_i - x_{i-1}} &= \norm{x_{i, 1} - x_{i-1, 1}}\\
    &\leq \sum_{j = 1}^{w - 1} \norm{x_{i-1, j+1} - x_{i-1, j}} + \norm{x_{i, 1}, x_{i-1, w}}.
    \end{aligned}
\end{equation*}

On the other hand, since $K_{t, 1} = K_{t, 2} = \cdots = K_{t, w} = K_t$, the offline optimal in game $\mathcal{I}'$ can pick $x_{t, 1}^* = x_{t, 2}^* = \cdots = x_{t, w}^*$. Since the offline optimal will not waste movement in duplicated bodies, if the offline optimal of instance $\mathcal{I}$ picks $x_t^* \in K_t$, picking $x_{t, 1}^* = \cdots = x_{t, w}^* = x_t^*$ will achieve the optimal cost in instance $\mathcal{I}'$. Therefore, we have
\begin{equation}\label{thm:Pred_not_help:e3}
    cost(OPT, \mathcal{I}') = cost(OPT, \mathcal{I}).
\end{equation}
Combining \eqref{thm:Pred_not_help:e1}, \eqref{thm:Pred_not_help:e2}, and \eqref{thm:Pred_not_help:e3}, we obtain that
$$cost(A', \mathcal{I}) \leq f(d, w)\cdot cost(OPT, \mathcal{I}).$$
Therefore, algorithm $A'$ has a competitive ratio of $f(d, w)$ in instance $\mathcal{I}$.
By the assumption made in the theorem, we see that
$$f(d, w) \geq g(d),$$
which completes the proof.
\end{proof}

\section{When do predictions help?}\label{sec:pred_helps}
The construction in the previous section highlights that it is not always possible for an online learner to leverage predictions to obtain near optimal cost, even in the case when costs are convex.  However, the positive results in the prior literature show that predictions can be powerful in many specific settings. Thus, a crucial question is: \emph{Under what general conditions is it possible for an online learner to leverage predictions to acheive near optimal cost?}

In this section, we introduce two general sufficient conditions that are motivated by the construction in the previous section and which ensure that the online learner can leverage predictions to achieve near-optimal cost.  Additionally, we present a new algorithm, Synchronized Fixed Horizon Control (SFHC), that can leverage predictions to achieve near optimal cost when these conditions hold. We then analyze SFHC in the sections that follow.

\subsection{Sufficient Conditions}

While there are many positive results in the literature that highlight specific conditions where it is possible for online algorithms to leverage predictions to achieve near-optimal cost, general sufficient conditions have not been presented.  Here, we introduce sufficient conditions that are general enough to contain many of the specific assumptions in previous results as special cases and that apply beyond online convex optimization to online non-convex optimization. Formally, the sufficient conditions we identify are the following.
\begin{description}
\item[Condition I: Order of Growth.] \emph{The hitting costs $f_t$ and movement cost $c$ satisfy\\
$ f_t(x) \geq \lambda \left(c(x, v_t) + c(v_t, x) \right)$ for all $x$,
where $v_t$ is a global minimum of $f_t$.}
\item[Condition II: Approximate Triangle Inequality.] \emph{The movement cost $c$ satisfies\\ $ c(x, z) \leq \eta \left( c(x, y) + c(y, z)\right)$ for all $x, y, z$.}
\end{description}

The first condition ensures that the hitting cost functions are not too flat around the minimizer $v_t$. This is useful since it helps limit the area where the offline optimal solution can be. Notice that the need for a condition of this type is motivated by the analysis of CBC in Section \ref{sec:limit_of_pred} and it is interesting to see that many previous papers in online convex optimization have assumptions that are special cases of this condition, e.g., \cite{chen2018smoothed, goel2018smoothed, goel2019beyond}. The second condition is an approximate form of the triangle inequality. Intuitively, without such a condition the cost for an online learner to ``catch up'' after making a mistake by moving in the wrong direction could be arbitrarily large, which would make it impossible to track the offline optimal in a way that maintains a constant competitive ratio.  

We would like to emphasize the generality of these conditions.  They capture many settings where previous papers have focused.  For example, the case of $\alpha-$polyhedral hitting costs that was studied in \cite{chen2018smoothed} corresponds to $\eta = 1$ and $\lambda = \frac{\alpha}{2}$. Similarly, the case of $m-$strongly convex hitting costs studied in \cite{goel2018smoothed, goel2019beyond} corresponds to $\eta = 2$ and $\lambda = \frac{m}{2}$.  Finally, the setting of geographical load balancing across data centers studied in \cite{lin2012online,lin2013dynamic, shi2019value} corresponds to $\eta = 1$ and $\lambda = \frac{1}{2}\min_s \frac{e_{0,s}}{\beta_s}$, where $e_0 = (e_{0,1}, \cdots, e_{0,d})$ is the cost of running different kinds of servers and $\beta = (\beta_1, \cdots, \beta_d)$ is the cost of starting different kinds of servers.  (This connection is not immediately obvious and so a proof is provided in Appendix \ref{Appendix:Geo}). 

\subsection{Synchronized Fixed Horizon Control} 

In order to show that the two conditions above are sufficient to allow a learner to leverage predictions, we introduce a new algorithm, Synchronized Fixed Horizon Control (SFHC), which we show has a $1+O(1/w)$ competitive ratio whenever the sufficient conditions hold -- regardless of whether hitting costs are convex or non-convex.  Thus, our results for SFHC apply more broadly than those for any existing algorithms. Further, our results show that SFHC achieves (nearly) the same performance bound as previous algorithms in settings where they do apply. 

SFHC is a variant of Averaging Fixed Horizon Control (AFHC), which was proposed in \cite{lin2012online} and has served as the basis for a number of improved algorithms in recent years, e.g., \cite{chen2016using, shi2019value}.  Like AFHC, SFHC works by combining the trajectories determined by $w$ different subroutines $SFHC(0), SFHC(1), \cdots, SFHC(w-1)$ (Algorithm \ref{alg:SFHC(h)}).  It either combines them deterministically (by averaging them) or in a randomized manner, leading to two variations: Deterministic SFHC (Algorithm \ref{alg:DSFHC}) and Randomized SFHC (Algorithm \ref{alg:RRFHC}).  Deterministic SFHC is sufficient for the case of convex hitting costs, but randomization is valuable when costs are non-convex. That randomization helps is perhaps surprising given that it has been proven that randomization does not help in smoothed online \textit{convex} optimization \cite{bansal20152}. 

To explain the workings of SFHC, we start with the case of $w=1$, i.e., where the online agent sees only the current cost function.  In this case, SFHC is ``greedy'' and picks $x_t = v_t := \argmin_x f_t(x)$, while AFHC, in contrast, picks $x_t$ that minimizes the sum of hitting cost and movement cost at timestep $t$, i.e. $x_t = \argmin_x \big(f_t(x) + c(x, x_{t-1})\big)$. SFHC with $w=1$ is a simple approach that is not optimal but, remarkably, is still competitive in many situations when Conditions I and II hold.  To understand why, consider an online agent whose goal is to choose $x_t$ to track the offline optimal point $x_t^*$, instead of simply minimizing costs.  It is impossible to exactly know $x_t^*$ ahead of time, even with predictions. However, the offline optimal point $x_t$ cannot be too far away from the minimizer $v_t$, since then it would incur significant costs; we can hence think of $v_t$ as an ``anchor" that keeps us close to the offline optimal. The Order of Growth condition controls exactly how close the offline point must be to $v_t$; the steeper the hitting costs, the more incentive for the offline optimal to stay close to $v_t$. The approximate triangle inequality property helps to bound the discrepancy caused by the fact that $v_t$ is only a proxy for $x_t^*$, not its exact location.  In particular, the approximate triangle inequality immediately gives the estimate $c(x_t, x_t^*) \leq \eta \left( c(x_t, v_t) + c(v_t, x_t^*) \right)$.


The key idea in SFHC for $w>1$ is to periodically ``sync up" with the greedy algorithm while simultaneously exploiting predictions. This guarantees that our solution cannot wander too far from the offline trajectory. This reasoning suggests the basic structure of the algorithm; at timestep $\tau$, a subroutine of SFHC chooses the sequence of points
\begin{subequations}\label{SFHC_intuition:e1}
    \begin{align}
    &\argmin_{x_{\tau:\tau+w-1}} \sum_{t=\tau}^{\tau+w-1} f_t(x_t) + c(x_t, x_{t-1})\label{SFHC_intuition:e1:objective}\\
    &\text{subject to }x_{\tau+w-1} = v_{\tau+w-1},\label{SFHC_intuition:e1:constraint}
    \end{align}
\end{subequations}
where $v_t$ is a global minimizer of $f_t$.  The constraint \eqref{SFHC_intuition:e1:constraint} in \eqref{SFHC_intuition:e1} is what leads to name \emph{Synchronized} Fixed Horizon Control.  It ensures that at the end of each fixed horizon the algorithm is constrained to make the greedy choice, i.e. it is periodically ``synchronized" with the greedy algorithm which attempts to track the offline optimal solution.  Clearly, this is strictly better than picking $x_t = v_t$ for all $t$ since it can use predictions to optimize among the trajectories that end at $v_{\tau+w-1}$. To contrast this with a related algorithm, note that AFHC also solves the optimization problem \eqref{SFHC_intuition:e1:objective} but it does not have constraint \eqref{SFHC_intuition:e1:constraint}.

Finally, as in AFHC, our algorithm maintains $w$ different subroutines performing the optimization in \eqref{SFHC_intuition:e1} separately and then combines them together to pick a point at time $t$. 

We now describe the SFHC algorithm more formally. The core piece of both Deterministic SFHC and Randomized SFHC is the SFHC$(h)$ subroutine. To define it, we need to introduce some notation first. Define $\Omega_h = \{k \mid k \equiv h \mod w, 0 \leq k \leq T\}$ and function $g_{\tau_1, \tau_2}: \mathbb{R}^{d \times (\tau_2 - \tau_1 - 1)} \to \mathbb{R}^+ \cup \{0\}$ ($0\leq \tau_1 < \tau_2 \leq T$) as:
\begin{equation}\label{equ:g_function_1}
    g_{\tau_1, \tau_2}(y) = \sum_{s=\tau_1 + 1}^{\tau_2} \left(f_s(y_s) + c(y_s, y_{s-1})\right),
\end{equation}
where the variable is $y = (y_{\tau_1 + 1},\cdots, y_{\tau_2 - 1}) \in \mathbb{R}^{d \times (\tau_2 - \tau_1 - 1)}$, $(y_{\tau_1}, y_{\tau_2})$ are defined as the minimizers $(v_{\tau_1}, v_{\tau_2})$. We can view $g_{\tau_1, \tau_2}$ as the total cost incurred between timestep $\tau_1+1$ and $\tau_2$, given the fixed choices of $y_{\tau_1}$ and $y_{\tau_2}$. Since the head and the tail of the subsequence of decision points $(y_{\tau_1}, \cdots, y_{\tau_2})$ are fixed to be the minimizers at the corresponding timesteps, only $y_{\tau_1+1}, \cdots, y_{\tau_2-1}$ can change freely in $\mathbb{R}^d$. In general, $SFHC(h)$ will minimize the function $g_{s, s+W}$ for $s \in \Omega_h$, except perhaps in the last prediction window, which may overshoot the end of the sequence of functions. To take this case into consideration, we need to extend the definition of function $g_{\tau_1, \tau_2}$ to include the case $\tau_2 > T$. If $\tau_2 > T$, $g_{\tau_1, \tau_2}: \mathbb{R}^{d \times (T - \tau_1)} \to \mathbb{R}^+ \cup \{0\}$ is defined as
\begin{equation}\label{equ:g_function_2}
    g_{\tau_1, \tau_2}(y) = \sum_{s=\tau_1 + 1}^{T} \left(f_s(y_s) + c(y_s, y_{s-1})\right),
\end{equation}
where the variable is $y = (y_{\tau_1 + 1},\cdots, y_{T}) \in \mathbb{R}^{d \times (T - \tau_1)}$ and $y_{\tau_1} = v_{\tau_1}$ is a fixed constant. Given this notation, $SFHC(h)$ is formally defined in Algorithm \ref{alg:SFHC(h)}.

%

To analyze $SFHC(h)$, it is useful to formulate it as an offline optimization. In particular, in phase $h$, $SFHC(h)$ outputs the solution of
\begin{equation}
    \begin{aligned}
    &\min_x \sum_{t=1}^T f_t(x_t) + c(x_t, x_{t-1})\\
    &\text{Subject to }x_t = v_t, \forall t \in \Omega_h.
    \end{aligned}
\end{equation}
It can be easily verified that this offline optimization can be implemented as an online algorithm with a prediction window $w$. 

Using, SFHC$(h)$, we can now define Deterministic SFHC. In short, Deterministic SFHC averages the decisions of the $w$ SFHC$(h)$ subroutines with equal weighting; see Algorithm \ref{alg:DSFHC} for the full details.  Notice that when Deterministic SFHC is required to commit $x_t$, hitting costs $f_t, f_{t+1}, \cdots, f_{t+w-1}$ have been revealed, so $SFHC(h)$ is able to decide its choice $x_t^{(h)}$. 

Randomized SFHC is a variation of Deterministic SFHC which picks one of the subroutines to follow uniformly at random instead of averaging the choices of the $w$ subroutines. We note that this choice is made exactly once, before the algorithm is run; Randomized SFHC does not resample from the subroutines once its initial random choice is made. See Algorithm \ref{alg:RRFHC} for details.

\begin{algorithm}[t]
\caption{SFHC with phase $h$: SFHC(h)}\label{alg:SFHC(h)}
\begin{algorithmic}[1]
\If{$h\geq 1$}
\State Pick $x_{1:h-1} = \argmin_y g_{0,h}(y)$.
\State Pick $x_h = v_h$.
\EndIf
\For{$t=h+1, \cdots, T$}
    \If{$t-1\in \Omega_h$}
    \If{$t+w-1\leq T$}
    \State Pick $x_{t:t+w-2} = \argmin_y g_{t-1, t+w-1}(y)$.
    \State Pick $x_{t+w-1} = v_{t+w-1}$.
    \Else
    \State Pick $x_{t:T} = \argmin_y g_{t-1, t+w-1}(y)$.
    \EndIf
    \EndIf
\EndFor
\end{algorithmic}
\end{algorithm}

\begin{algorithm}[t]
\caption{Deterministic SFHC}\label{alg:DSFHC}
\begin{algorithmic}[1]
\For{$t=1, \cdots, T$}
\State Suppose $x_t^{(h)}$ is the point picked by $SFHC(h)$ at timestep $t$.
\State Commit $x_t = \frac{1}{w}\sum_{h=0}^{w-1}x_t^{(h)}$.
\EndFor
\end{algorithmic}
\end{algorithm}

\begin{algorithm}[t]
\caption{Randomized SFHC (Version A)}\label{alg:RRFHC}
\begin{algorithmic}[1]
\State Choose $h$ uniform randomly from $\{0, 1, \cdots, w - 1\}$.
\State Run $SFHC(h)$ to determine $x_1, x_2, \cdots, x_T$.
\end{algorithmic}
\end{algorithm}

\section{Convex hitting costs}\label{sec:convex}

To show that the conditions presented in Section \ref{sec:pred_helps} are sufficient for SFHC to be competitive, we start by focusing on the case of convex costs. While our goal in the paper is to provide results for non-convex costs, we present the convex setting first because the structure of the analysis in the convex case serves as the basis of the proofs in the non-convex case, with the bulk of the proof applying to both the convex and non-convex cases. The proof in the convex case thus highlights exactly where additional complexity is needed for the analysis of non-convex costs. 

The following theorem highlights that, in the convex setting, SFHC achieves a competitive ratio that matches the order of the best known bounds for many previously known algorithms, such as Online Balanced Descent (see \cite{chen2018smoothed, goel2018smoothed}) and Averaging Fixed Horizon Control (see \cite{lin2012online}), while applying more generally than any previous algorithm.

\begin{theorem}\label{thm:ARFHC_general_cr_1}
Consider an online optimization problem with movement and hitting costs that satisfy Conditions I and II. Suppose the hitting cost functions and the movement cost function are convex. 
\begin{enumerate}[(i)] 
    \item Deterministic SFHC has a competitive ratio of $\max \Big(1 + \frac{\eta + \eta^2}{2\lambda}, \eta^2\Big)$ when it has access to $w = 1$ predictions.
    \item Deterministic SFHC has a competitive ratio of \\ $\left( 1 + \frac{1}{w}\max\left(\frac{\eta}{\lambda}, 2(\eta - 1) \right)\right)$ when it has access to $w \geq 2$ predictions.
\end{enumerate} 
\end{theorem}

Some insight for the form of the bounds in the theorem above follows from thinking about $\lambda$ and $\eta$.  
Intuitively, as $\lambda$ becomes smaller, the hitting cost functions become more flat, which makes it harder for the online agent to estimate and track the offline optimal points. For example, the Convex Body Chasing problem has $\lambda = 0$ and, as we discuss in Section \ref{sec:limit_of_pred}, this represents a difficult subclass for the online agent. On the other hand, when $\eta$ becomes large, the offline adversary, with the full knowledge of the future, can obtain more advantage over the online agent by dividing a single large jump to multiple small steps. The lower bound proved for SOCO with $\ell_2$ squared movement cost in \cite[Theorem 1]{goel2019beyond} highlights this intuition.

To provide additional context for Theorem \ref{thm:ARFHC_general_cr_1} it is useful to compare it to the special cases that have been studied previously in the literature.  
The key contrast with the previous literature is that SFHC provides a competitive bound in settings much more general than previous results.  To highlight this, a first example is the $\alpha-$polyhedral problem setting considered in \cite{chen2018smoothed}, which corresponds to setting $\eta = 1, \lambda = \frac{\alpha}{2}$.  In this setting, the competitive ratio for Online Balanced Descent (OBD) proved in \cite{chen2018smoothed} is $3 + \frac{8}{\alpha}$ for $w = 1$. Theorem \ref{thm:ARFHC_general_cr_1} provides a strictly strongly and more general result.  It guarantees that Deterministic SFHC is $\left(1 + \frac{2}{w\alpha}\right)$-competitive for all $w > 0$. Another important setting that has previously received attention is the the $m-$strongly convex problem setting considered in \cite{goel2018smoothed, goel2019beyond}, which corresponds to $\eta = 2, \lambda = \frac{m}{2}$. In this context, the best known competitive ratio is $\frac{1}{2}\left(1 + \sqrt{1 + \frac{4}{m}}\right)$, achieved by Regularized Online Balanced Descent (ROBD) in the context of $w=1$ \cite{goel2019beyond}. While our result does not match the performance of ROBD for the case of $w=1$, the competitive ratio of Deterministic SFHC given by Theorem \ref{thm:ARFHC_general_cr_1} is $1 + \frac{1}{w}\max \left(\frac{4}{m}, 2\right)$ when $w > 1$ and, to the best of our knowledge, this is the best known result in this setting.  Importantly, the previous algorithms and analysis in both of these settings are tuned to the details of the setting and do not apply more broadly.  In contrast, the bound on the competitive ratio of Deterministic SFHC applies much more generally, i.e., whenever Conditions I and II hold.

We end this section by proving Theorem \ref{thm:ARFHC_general_cr_1}.  The bulk of our proof does not require the assumption that hitting and movement costs are convex.  This is important because it means that a large fraction of the argument can be used in the context of non-convex costs, which is our focus in Section \ref{sec:non_convex}.  To highlight this fact, we organize the proof into a set of lemmas, and then apply these lemmas to prove the theorem. 

Our first lemma focuses on Case (i) in Theorem \ref{thm:ARFHC_general_cr_1}, i.e., when $w=1$. The result does not require convexity.

\begin{lemma}\label{lemma:ARFHC_general_cr_w1}
Consider an online optimization problem with movement and hitting costs that satisfy Conditions I and II. Deterministic SFHC has a competitive ratio of $\max \left(1 + \frac{\eta + \eta^2}{2\lambda}, \eta^2\right)$ when it has access to $w = 1$ predictions.
\end{lemma}

\begin{proof} Since SFHC picks the minimizer $v_t$ of hitting cost function $f_t$ at timestep $t$,  the hitting cost incurred by the online agent at timestep $t$ is $f_t(v_t)$ and the movement cost is $c(v_t, v_{t-1})$. We can upper bound $c(v_t, v_{t-1})$ in the following two symmetric ways by the Approximate Triangle Inequality (Condition II):
\begin{equation}\label{thm:GREEDY_e0_1}
\begin{aligned}
    c(v_t, v_{t-1})\leq{}& \eta c(v_t, x_t^*) + \eta c(x_t^*, v_{t-1})\\
    \leq{}& \eta c(v_t, x_t^*) + \eta^2 c(x_t^*, x_{t-1}^*) + \eta^2 c(x_{t-1}^*, v_{t-1}),
\end{aligned}
\end{equation}
and
\begin{equation}\label{thm:GREEDY_e0_2}
\begin{aligned}
    c(v_t, v_{t-1})\leq{}& \eta c(v_t, x_{t-1}^*) + \eta c(x_{t-1}^*, v_{t-1})\\
    \leq{}& \eta^2 c(v_t, x_t^*) + \eta^2 c(x_t^*, x_{t-1}^*) + \eta c(x_{t-1}^*, v_{t-1}).
\end{aligned}
\end{equation}
Adding \eqref{thm:GREEDY_e0_1} and \eqref{thm:GREEDY_e0_2} together, we obtain that
$$c(v_t, v_{t-1}) \leq \frac{\eta^2 + \eta}{2} c(v_t, x_t^*) + \eta^2 c(x_t^*, x_{t-1}^*) + \frac{\eta^2 + \eta}{2} c(x_{t-1}^*, v_{t-1}).$$
Recalling that $v_t$ is the global minimum of $f_t$ we obtain that
\begin{equation}\label{thm:GREEDY_e3}
\begin{aligned}
    &f_t(v_t) + c(v_t, v_{t-1})\\
    \leq{}& f_t(x_t^*) + \frac{\eta^2 + \eta}{2} c(v_t, x_t^*) + \eta^2 c(x_t^*, x_{t-1}^*) + \frac{\eta^2 + \eta}{2} c(x_{t-1}^*, v_{t-1}).
\end{aligned}
\end{equation}
Next, summing \eqref{thm:GREEDY_e3} over timesteps $t = 1, \cdots, T$, we can compute
\begin{subequations}\label{thm:GREEDY_e4}
\begin{align}
    &\sum_{t=1}^T f_t(v_t) + c(v_t, v_{t-1})\nonumber\\
    \leq{}& \sum_{t=1}^T f_t(x_t^*) + \frac{\eta^2 + \eta}{2} \sum_{t=1}^T c(v_t, x_t^*)\nonumber\\
    &+ \eta^2 \sum_{t=1}^T c(x_t^*, x_{t-1}^*) + \frac{\eta^2 + \eta}{2} \sum_{t=1}^T c(x_{t-1}^*, v_{t-1})\nonumber\\
    \leq{}& \sum_{t=1}^T f_t(x_t^*) + \eta^2 \sum_{t=1}^T c(x_t^*, x_{t-1}^*) + \frac{\eta^2 + \eta}{2}\sum_{t=1}^T \left(c(v_t, x_t^*) + c(x_t^*, v_t) \right)\nonumber\\
    \leq{}& \sum_{t=1}^T f_t(x_t^*) + \eta^2 \sum_{t=1}^T c(x_t^*, x_{t-1}^*) + \frac{\eta^2 + \eta}{2\lambda}\sum_{t=1}^T f_t(x_t^*)\label{thm:GREEDY_e4:s1}\\
    ={}& \left(1 + \frac{\eta^2 + \eta}{2\lambda}\right)\sum_{t=1}^T f_t(x_t^*) + \eta^2 \sum_{t=1}^T c(x_t^*, x_{t-1}^*),\nonumber
\end{align}
\end{subequations}
where we use Condition II in \eqref{thm:GREEDY_e4:s1}.
\end{proof}

Next, we prove a lemma that bounds the average total cost across the $w$ subroutines of SFHC.  Again, this lemma does not require the assumption of convexity.  Thus, it serves as the basis not just for the result in Theorem \ref{thm:ARFHC_general_cr_1}, but also for our analysis in Section \ref{sec:non_convex} of non-convex costs.

\begin{lemma}\label{lemma:ARFHC_general_cr}
Consider an online optimization problem with movement and hitting costs that satisfy Conditions I and II.  The average total cost of the $w$ subroutines of SFHC, i.e., the arithmetic mean of\\ $\{cost\left(SFHC(h)\right), h = 0, \cdots, w-1\}$, is upper bounded by 
    $$\left( 1 + \frac{1}{w}\max\left(\frac{\eta}{\lambda}, 2(\eta - 1) \right)\right)\cdot cost(OPT)$$
    given access to $w \geq 2$ predictions.
\end{lemma}

\begin{figure}
\begin{center}
    \includegraphics{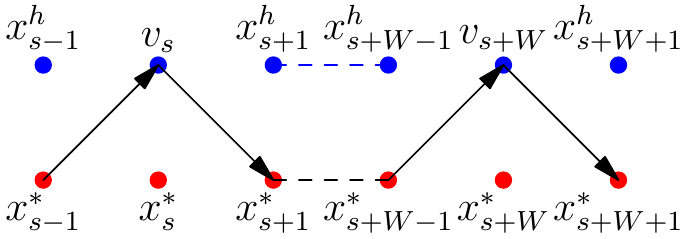}
\end{center}
\caption{\emph{\textbf{Illustration of the Proof of Theorem \ref{thm:ARFHC_general_cr_1}.} By the definition of $SFHC(h)$, we know the blue point sequence (picked by $SFHC(h)$) is better than the black path (offline optimal solution with $x_s^*$ substituted by $v_s$, for all $s \in \Omega_h$). The task is to compare the black path with the red point sequence (the actual offline optimal).}}
\label{figure:DSFHC_Diagram}
\end{figure}

\begin{proof}[Proof of Lemma \ref{lemma:ARFHC_general_cr}]
To begin, we define some notation.  For a point sequence $x = (x_1, x_2, \cdots, x_T) \in \mathbb{R}^{d \times T}$, we use $x_{\tau: t} \in \mathbb{R}^{d \times (t - \tau + 1)}$ to denote $x$'s subsequence $(x_\tau, x_{\tau + 1}, \cdots, x_t).$
Also, recall that $C_T(x) = \sum_{t=1}^T \left(f_t(x_t) + c(x_t, x_{t-1})\right),$ is the total cost of the sequence $x$. We use $H_t^* = f_t(x_t^*)$ and $M_t^* = c(x_t^*, x_{t-1}^*)$ to denote the offline optimal hitting cost and movement cost incurred at timestep $t$.  Finally, recall that we can formulate $SFHC(h)$ as an offline optimization
\begin{equation}
    \begin{aligned}
    &\min_x \sum_{t=1}^T f_t(x_t) + c(x_t, x_{t-1})\\
    &\text{Subject to }x_t = v_t, \forall t \in \Omega_h.
    \end{aligned}
\end{equation}
In general, $SFHC(h)$ breaks the sequence of actions up into subsequences of length $w-1$ by fixing $x_t^h = v_t$ for $t \in \Omega_h$. It is able to minimize each subsequence separately because they are independent from each other due to the synchronization enforced by the algorithm. However, the first and the last subsequences need additional attention because their length may be less than $w-1$.

Recall that $SFHC(h)$ (defined in Algorithm \ref{alg:SFHC(h)}) selects the minimizers of function $g_{s, s+w}$ as its choice $x_{s+1}^h, \cdots, x_{s+w-1}^h$ for $s\in \Omega_h$ and $x_{s+w}^h = v_{s+w}$ if $s + w \leq T$. The definition of the function $g$ can be found in Section \ref{sec:pred_helps}. 
For all $s \in \Omega_h$, since $x_{s+1: s+w-1}^h$ are the minimizers of $g_{s, s+w}$, we have that
\begin{equation}\label{thm:RAFHC_general_cr_1:e1}
    g_{s, s + w}(x_{s + 1: s + w - 1}^h) \leq g_{s, s + w}(x_{s + 1: s + w - 1}^*),
\end{equation}
where $x^h\in \mathbb{R}^{d\times T}$ is the solution picked by $SFHC(h)$ and $x^*\in \mathbb{R}^{d\times T}$ is the offline optimal solution. Similarly, we also have that
\begin{equation}\label{thm:RAFHC_general_cr_1:e1_1}
    g_{0, h}(x_{1: h-1}^h) \leq g_{0, h}(x_{1: h-1}^*).
\end{equation}
Summing up \eqref{thm:RAFHC_general_cr_1:e1} over $s \in \Omega_h$ together with \eqref{thm:RAFHC_general_cr_1:e1_1}, we can upper bound the total cost of $SFHC(h)$ by
\begin{equation}\label{thm:RAFHC_general_cr_1:e1_2}
    \begin{aligned}
    C_T(x^h) &= g_{0, h}(x_{1: h-1}^h) + \sum_{s \in \Omega_h}g_{s, s + w}(x_{s + 1: s + w - 1}^h)\\
    &\leq g_{0, h}(x_{1: h-1}^*) + \sum_{s \in \Omega_h}g_{s, s + w}(x_{s + 1: s + w - 1}^*).
    \end{aligned}
\end{equation}
Equation \eqref{thm:RAFHC_general_cr_1:e1_2} highlights that 
if we substituted $x_s^*$ by $v_s$ for all $s \in \Omega_h$ in the offline optimal solution, the resulting sequence (which satisfies the synchronization constraint) is no better than the solution picked by $SFHC(h)$. However, the actual offline optimal is not subject to the synchronization constraint. 

Now we compare the upper bound in \eqref{thm:RAFHC_general_cr_1:e1_2} with the actual offline optimal cost. An illustration is given in Figure \ref{figure:DSFHC_Diagram}. We see that
\begin{equation}\label{thm:RAFHC_general_cr_1:e1_3}
    \begin{aligned}
    C_T(x^h) \leq{}& g_{0, h}(x_{1: h-1}^*) + \sum_{s \in \Omega_h}g_{s, s + w}(x_{s + 1: s + w - 1}^*)\\
    ={}& C_T(x^*) - \sum_{s \in \Omega_h}f_s(x_s^*) + \sum_{s \in \Omega_h}f_s(v_s)\\
    &+ \sum_{s \in \Omega_h\cap [T-1]} \left(c(x_{s+1}^*, v_{s}) - c(x_{s+1}^*, x_{s}^*)\right)\\
    &+ \sum_{s \in \Omega_h\cap [T]} \left(c(v_{s}, x_{s - 1}^*) - c(x_{s}^*, x_{s - 1}^*)\right).
    \end{aligned}
\end{equation}

\noindent Since $v_s = \argmin_x f_s(x)$, we have
\begin{equation}\label{thm:RAFHC_general_cr_1:e1_4}
    f_s(v_s) \leq f_s(x_s^*), \forall s \in \Omega_h.
\end{equation}
Conditions I and II guarantee that for all $s \in \Omega_h$
\begin{equation}\label{thm:RAFHC_general_cr_1:e1_5}
    c(x_{s+1}^*, v_{s}) \leq \eta \left(c(x_{s+1}^*, x_{s}^*) + c(x_{s}^*, v_{s})\right),
\end{equation}
and
\begin{equation}\label{thm:RAFHC_general_cr_1:e1_6}
    \begin{aligned}
    c(v_{s}, x_{s - 1}^*)\leq{}& \eta \left(c(x_{s}^*, x_{s - 1}^*) + c(v_{s}, x_{s}^*)\right)\\
    \leq{}& \eta \left(c(x_{s}^*, x_{s - 1}^*) + \frac{1}{\lambda}f_{s}(x_{s}^*) - c(x_{s}^*, v_{s})\right).
    \end{aligned}
\end{equation}
Substituting \eqref{thm:RAFHC_general_cr_1:e1_4}, \eqref{thm:RAFHC_general_cr_1:e1_5} and \eqref{thm:RAFHC_general_cr_1:e1_6} into \eqref{thm:RAFHC_general_cr_1:e1_3}, we see that
\begin{equation}\label{thm:RAFHC_general_cr_1:e1_7}
    \begin{aligned}
    &C_T(x^h) - C_T(x^*)\\
    \leq{}& (\eta - 1)\sum_{s \in \Omega_h\cap [T-1]}c(x_{s+1}^*, x_{s}^*) + \eta \sum_{s \in \Omega_h\cap [T-1]}c(x_{s}^*, v_s)\\
    &+ (\eta - 1)\sum_{s \in \Omega_h\cap [T]} c(x_{s}^*, x_{s - 1}^*) + \frac{\eta}{\lambda}\sum_{s \in \Omega_h\cap [T]} f_s(x_s^*)\\
    &- \eta \sum_{s \in \Omega_h\cap [T]}c(x_{s}^*, v_s)\\
    \leq{}& \frac{\eta}{\lambda}\sum_{s \in \Omega_h}H_s^* + (\eta - 1)\sum_{s \in \Omega_h}\left(M_s^* + M_{s+1}^*\right),
    \end{aligned}
\end{equation}
where we assume $M_{T+1}^* = 0$ without loss of generality.

Recall that we use $x^h = (x_1^h, \cdots, x_T^h)$ to denote the point sequence picked by $SFHC(h)$. The averaging cost incurred by all $w$ subroutines satisfies that
\begin{subequations}\label{thm:RAFHC_general_cr_1:e8}
    \begin{align}
    &\frac{1}{w}\sum_{h = 1}^{w} C_T(x^h)\nonumber\\
    \leq{}& \frac{1}{w}\sum_{h = 1}^{w} \left( C_T(x^*) + \frac{\eta}{\lambda}\sum_{s \in \Omega_h}H_s^* + (\eta - 1)\sum_{s \in \Omega_h}\left(M_s^* + M_{s+1}^*\right) \right)\label{thm:RAFHC_general_cr_1:e8:s1}\\
    \leq{}& C_T(x^*) + \frac{1}{w}\cdot \frac{\eta}{\lambda}\sum_{h = 1}^{w} \sum_{s \in \Omega_h}H_s^* + \frac{\eta - 1}{w}\sum_{h = 1}^{w} \sum_{s \in \Omega_h}\left(M_s^* + M_{s+1}^*\right)\nonumber\\
    \leq{}& cost(OPT) + \frac{1}{w}\cdot \frac{\eta}{\lambda}\sum_{t=1}^T H_t^* + \frac{2(\eta - 1)}{w}\sum_{t=1}^T M_t^*,\label{thm:RAFHC_general_cr_1:e8:s2}
    \end{align}
\end{subequations}
where in \eqref{thm:RAFHC_general_cr_1:e8:s1} we use \eqref{thm:RAFHC_general_cr_1:e1_7}; in \eqref{thm:RAFHC_general_cr_1:e8:s2} we use $cost(OPT)$ to denote $C_T(x^*)$.

\end{proof}

While Lemma \ref{lemma:ARFHC_general_cr} guarantees good average performance across the $w$ subroutines of $SFHC$, the question it does not answer is how to combine different phases together into a single algorithm. This question turns out to be non-trivial and needs to be answered differently depending on properties of the cost function. 

In the case of convex costs, the most direct way to convert Lemma \ref{lemma:ARFHC_general_cr} to a competitive ratio result is to apply Jensen's inequality. This approach requires the convexity of both the hitting and movement cost functions and completes our proof of Theorem \ref{thm:ARFHC_general_cr_1}.  However, in the case of non-convex costs, more work is required.  That is the focus of Section \ref{sec:non_convex}.

\begin{proof}[Proof of Theorem \ref{thm:ARFHC_general_cr_1}]
Case (i) follows directly from Lemma \ref{lemma:ARFHC_general_cr_w1}.  Thus, we only need to consider Case (ii).  Since the movement cost function $c$ and every hitting cost function $f_t$ is convex, we know that the total function $C_T$ is convex. Therefore, by Jensen's inequality, we obtain from Lemma \ref{lemma:ARFHC_general_cr} that the cost of Deterministic SFHC satisfies
\begin{equation*}
\begin{aligned}
    cost(DSFHC) &\leq \frac{1}{w}\sum_{h = 1}^{w} cost(SFHC(h))\\
    &\leq \left( 1 + \frac{1}{w}\max\left(\frac{\eta}{\lambda}, 2(\eta - 1) \right)\right)\cdot cost(OPT),
\end{aligned}
\end{equation*}
which completes the proof.  
\end{proof}
\section{Non-convex hitting costs}\label{sec:non_convex}

The main contribution of this paper is to give the first algorithm that leverages predictions to obtain a dimension-free competitive ratio for smoothed online \emph{non-convex} optimization.  To the best of our knowledge, no prior algorithms achieve a dimension-free competitive ratio for non-convex hitting costs.
For example, OBD (see \cite{chen2018smoothed, goel2018smoothed}) is the state-of-the-art algorithm for the convex-case without predictions and it uses properties of the intermediate geometry of the convex hitting cost functions which do not hold for non-convex functions.  Therefore, it does not directly extend to the non-convex setting.

The sufficient conditions we have identified in Section \ref{sec:pred_helps} are the key to enabling our analysis in the non-convex setting.  Specifically, the sufficient conditions do not rely on convexity of the hitting cost functions and so apply more broadly then online convex optimization.  However, the proof of Theorem \ref{thm:ARFHC_general_cr_1} for the convex case relies on convexity through its use of Jensen's inequality, which requires the total function $C_T: \mathbb{R}^{d\times T} \rightarrow \mathbb{R}_{\geq 0}$ (as defined in \eqref{eq:total_cost}) to be convex. This step of the proof is needed due to the fact that Deterministic SFHC averages the different subroutines.  Thus, extending the result to the case of non-convex hitting costs requires finding an alternative to Jensen's inequality.

In this section, we present two results that take two different paths to replacing Jensen's inequality in the proof.  The first uses a generalization of Jensen's inequality that can be applied to non-convex cost functions.  In particular, if the non-convex function is easily ``convexifiable'' then Deterministic SFHC provides $C + O(1/w)$, for a dimension-free constant $C$.  The second approach leads to a better performance bound. In particular, by using randomization to replace the averaging in SFHC we obtain a result that applies broadly to non-convex functions, yielding an expected competitive ratio for an oblivious adversary that is $1+O(1/w)$, as desired.  Further, with a bit more complexity in the form of the randomness, the algorithm can be adapted to perform well against a semi-adaptive adversary.  



Throughout this section, we focus on the challenges associated with the \textit{online} component of the problem, rather than the challenges associated with solving a non-convex optimization efficiently.  Thus, we assume that it is possible to solve the optimization problems involved in SFHC and achieve a global optimizer, even though they are non-convex. Under some reasonable assumptions, we can solve some specific non-convex optimization problems efficiently. We refer interested readers to \cite{jain2017non} for a survey. In practice, some error will result from limited computation time or using heuristic techniques to solve the non-convex optimization and that will lead to an additional cost degradation that depends on the technique used and the form of the non-convex functions.  It will be interesting to bound the interaction between the approximation error of non-convex solvers and the competitive ratio of SFHC for specific heuristics and classes of non-convex functions in future work, but that is beyond the scope of the current paper.

\subsection{Deterministic SFHC}\label{subsec:det}

The first approach we use to move beyond convex hitting costs is to focus on hitting cost functions that are ``convexifiable.'' For such functions, we can hope to obtain a dimension-free competitive ratio using the \emph{same} algorithm as in the convex case: Deterministic SFHC.  

To analyze Deterministic SFHC in this setting, we need to introduce the notion of a \textit{convexifiable function}. Intuitively, the \textit{convexifier} of a convexifiable function can be viewed as the "distance" between this function and the convexity property.  Formally, we have the following definition.

\begin{definition}\label{def:convexifiable}
Suppose $C$ is convex set in $\mathbb{R}^d$. If a function $f: C \to \mathbb{R}$ satisfies that $\exists \alpha \in \mathbb{R}, s.t. \phi(x) = f(x) - \frac{\alpha}{2}\norm{x}_2^2$ is convex, $f$ is \textbf{convexifiable} and $\alpha$ is a \textbf{convexifier} of function $f$.
\end{definition}

This definition is useful because it allows us to use a generalization of Jensen's inequality, which was a crucial tool in our analysis in the convex setting.  The following proposition is Theorem 1 in \cite{zlobec2004jensen}.

\begin{proposition}[Jensen's inequality for convexifiable functions]\label{prop:Jensen_convexifiable} Consider a convexifiable function $f:\mathbb{R}^d \to \mathbb{R}$ on a bounded nontrivial convex set $C$ of $\mathbb{R}^d$ and its convexifier $\alpha$. Then
$$f\left(\sum_{i=1}^p \lambda_i x_i\right) \leq \sum_{i=1}^p \lambda_i f(x_i) - \frac{\alpha}{2}\left(\sum_{i,j = 1, i < j}^p \lambda_i \lambda_j \norm{x_i - x_j}_2^2 \right)$$
for every set of $p$ points $x_i, i = 1, \cdots, p$ in $C$ and all real scalars $\lambda_i \geq 0, i = 1, \cdots, p$, with $\sum_{i=1}^p \lambda_i = 1$.
\end{proposition}

Using this generalized Jensen's inequality as a tool, we can prove the following theorem, which bounds the competitive ratio of Deterministic SFHC for non-convex hitting costs.

\begin{theorem}\label{thm:AFHC_non_convex_1}
Consider an online optimization problem where the movement and hitting costs satisfy Conditions I and II.  Suppose additionally that $f_t$ is convexifiable and its convexifier is upper bounded by $\alpha > 0$ and that $c(x_t, x_{t-1}) = \frac{1}{2}\norm{x_t - x_{t-1}}_2^2.$ Under this assumption, we have $\eta = 2$.
\begin{enumerate}[(i)] 
    \item Deterministic SFHC has a competitive ratio of $\max \left(1 + \frac{3}{\lambda}, 4\right)$ when it has access to $w = 1$ predictions.
    \item Deterministic SFHC has a competitive ratio of $$\left(1 + \frac{\alpha}{\lambda}\right)\cdot \left( 1 + \frac{1}{w}\max\left(\frac{2}{\lambda}, 2 \right)\right)$$
    when it has access to $w \geq 2$ predictions.
\end{enumerate} 
\end{theorem}


This theorem represents the first dimension-free competitive ratio for online non-convex optimization with switching costs.  Unlike our results in the convex setting, the result is restricted to squared $\ell_2$ movement costs, for technical reasons (see Lemma \ref{lem:AFHC_non_convex_1}).  Another key contrast with our results in the convex setting is that the competitive ratio does not converge to $1$ as $w\to\infty$.  These two limitations, combined with the fact that the competitive ratio is small only for hitting costs that are ``nearly'' convex, motivate us to consider randomized algorithms in the next section.  Through the use of randomization, we are able to eliminate all these restrictions (see Theorem \ref{thm:nonconvex_rsfhc}).

While the full proof of Theorem \ref{thm:AFHC_non_convex_1} is deferred to Appendix \ref{Appendix:AFHC_non_convex_1}, a key step in the proof is the following lemma, which gives a bound on the loss of the averaging step in the presence of non-convexity and highlights why the limitation to squared $\ell_2$ movement costs is present. Note that, as expected, the loss grows proportionally with the convexifier of the total cost function.  

\begin{lemma}\label{lem:AFHC_non_convex_1}
Consider an online optimization problem where the hitting cost functions satisfy $f_t(x) \geq \lambda \norm{x - v_t}_2^2$ and the movement cost is given by $c(x_t, x_{t-1}) = \frac{1}{2}\norm{x_t - x_{t-1}}_2^2.$
Suppose $\forall t$, $-\alpha$ is a convexifier of $f_t$. We have
$$cost(DSFHC) \leq \frac{1}{w}\left(1 + \frac{\alpha}{\lambda} \right)\sum_{i=1}^w cost\left(SFHC(h)\right).$$
\end{lemma}

\begin{proof}[Proof of Lemma \ref{lem:AFHC_non_convex_1}]
Throughout the proof, we use $x^h \in \mathbb{R}^{d\times T}$ to denote the point sequence picked by $SFHC(h)$ and use $x \in \mathbb{R}^{d\times T}$ to denote the point sequence picked by $DSFHC$. Applying Jensen's inequality for convexifiable functions (Proposition \ref{prop:Jensen_convexifiable}), we see that the cost incurred by DSFHC can be upper bounded by the average of $\{cost\left(SFHC(h)\right), h = 0, \cdots w-1\}$ plus an additive term related to the distance between each pair of SFHCs. Specifically, we have that
\begin{equation}\label{thm:AFHC_non_convex_1:eq1}
    C_T(x) \leq \frac{1}{w}\sum_{h=1}^w C_T(x^h) + \frac{\alpha}{4}\sum_{i, j = 1}^w \frac{1}{w^2}\norm{x^i - x^j}_2^2.
\end{equation}
Using $v = (v_1, v_2, \cdots, v_T)$ to denote the sequence of minimizers, we can bound the additive term by
\begin{subequations}\label{thm:AFHC_non_convex_1:eq2}
    \begin{align}
    \sum_{i, j = 1}^w \frac{1}{w^2}\norm{x^i - x^j}_2^2
    &\leq \frac{1}{w^2}\sum_{i, j = 1}^w \left(\norm{x^i - v} + \norm{x^j - v} \right)^2\label{thm:AFHC_non_convex_1:eq2:s1}\\
    &\leq \frac{2}{w^2}\sum_{i, j = 1}^w \left(\norm{x^i - v}^2 + \norm{x^j - v}^2 \right)\label{thm:AFHC_non_convex_1:eq2:s2}\\
    &={} \frac{4}{w} \sum_{i=1}^w \norm{x^i - v}^2\nonumber\\
    &\leq{} \frac{4}{\lambda w} \sum_{i=1}^w C_T(x^i),\label{thm:AFHC_non_convex_1:eq2:s3}
    \end{align}
\end{subequations}
where  we use the triangle inequality in \eqref{thm:AFHC_non_convex_1:eq2:s1} and we use the generalized mean inequality in \eqref{thm:AFHC_non_convex_1:eq2:s2}. Substituting \eqref{thm:AFHC_non_convex_1:eq2} into \eqref{thm:AFHC_non_convex_1:eq1}, we see that $$C_T(x) \leq \frac{1}{w}\left(1 + \frac{\alpha}{\lambda} \right)\sum_{h=1}^w C_T(x^h),$$ which completes the proof.
\end{proof}

\subsection{Randomized SFHC}\label{subsec:rand}

While it is known that randomization cannot lead to better competitive ratio in the case of convex hitting costs \cite{bansal20152}, we use randomization in the non-convex case to derive an improved competitive ratio. Specifically, we use a simple but powerful idea to directly extend the analysis in the convex setting to the non-convex case.  In order to bypass the challenge of applying Jensen's inequality, in Randomized SFHC (see Algorithm \ref{alg:RRFHC}) we use randomness to pick a subroutine $h$ uniformly at random from the set $\{0, 1, \cdots, w - 1\}$ and then run $SFHC(h)$. This simple idea works very well, and with a small modification of the argument in the convex case, we obtain the following theorem.

\begin{theorem} \label{thm:nonconvex_rsfhc}
Consider an online optimization problem where the movement and hitting costs satisfy Conditions I and II. 
\begin{enumerate}[(i)]
\item Randomized SFHC (Version A) has a competitive ratio of  $$\max \left(1 + \frac{\eta + \eta^2}{2\lambda}, \eta^2\right)cost(OPT),$$ given $w = 1$ predictions.
\item Randomized SFHC (Version A) has an expected cost that is bounded in terms of the offline optimal as follows $$\mathbb{E}\left(cost(ALG)\right) \leq \left( 1 + \frac{1}{w}\max\left(\frac{\eta}{\lambda}, 2(\eta - 1) \right)\right) cost(OPT),$$
given a oblivious adversary and $w \geq 2$ predictions.
\end{enumerate}
\end{theorem}

Note that the case of $w=1$ follows immediately from Lemma \ref{lemma:ARFHC_general_cr_w1}, and in this case there is no randomization in the algorithm.  The case of $w>1$, on the other hand, requires randomness.  Our result in this case considers an oblivious adversary, which means that the adversary must specify the sequence of hitting cost functions independently from the algorithm's decisions, i.e., without knowing the random sample of the algorithm.  This is a common adversarial model in online algorithm, e.g., see \cite{audibert2010regret, bubeck2012regret}, and captures the fact that in real-world applications the environment is not able to react to the actions taken by the algorithm.  Note that the expectation in the theorem is taken over the randomness in Randomized SFHC (Version A) and that optimal sequence and cost are not impacted by the randomness of the algorithm.


The theorem highlights that the expected cost incurred by Randomized SFHC satisfies the same upper bound given in Theorem \ref{thm:ARFHC_general_cr_1}, without the requirement that the hitting costs to be convex. Further, Randomized SFHC avoids the requirement that the hitting cost function be nearly convex and so applies more broadly than Deterministic SFHC. 

\begin{proof}  
Suppose the hitting cost sequence is fixed by the adversary before the learner begins.  Then, since the subroutine $h$ is picked uniform randomly from the set $\{0, 1, \cdots, w-1\}$, we have that
\begin{equation}
    \mathbb{E}\left(cost(RSFHC)\right) = \frac{1}{w}\sum_{h = 1}^{w} cost(x^h).
\end{equation}
Given this, the theorem follows directly from Lemma \ref{lemma:ARFHC_general_cr}.
\end{proof}

Theorem \ref{thm:nonconvex_rsfhc} is already a strong result about the performance of Randomized SFHC; however, it is interesting to understand whether it is possible to strengthen the result further in order to consider a stronger adversary. The proof highlights that the sequence of hitting costs must be fixed before the algorithm makes its random choice since, otherwise, after learning the subroutine $h$ that is picked by the algorithm, the adversary could design the future hitting cost sequence so that $\sum_{t \in \Omega_h} (H_t^* + M_t^*) >> \frac{1}{w}\sum_{t=1}^T (H_t^* + M_t^*)$. In this case, the expected competitive ratio upper bound would be considerably worse.

For practical settings, it is natural to focus on the performance of an online learner with respect to an oblivious adversary; however it is interesting from a theoretical perspective to understand if it is possible for an online learner to still perform well against a non-oblivious adversary.  In order for Randomized SFHC to perform well against a non-oblivious adversary, it needs to incorporate more randomness.  Instead of only making one random choice at the beginning of the instance, randomness needs to be used throughout the instance. 

We adjust Randomized SFHC along those lines in (Algorithm \ref{alg:RRFHC_v2}).  In Version B of the algorithm, we use randomness to choose the timesteps at which we ``synchronize'', i.e., set $x_t = v_t$. By ensuring that the expected distance between two such timesteps is lower bounded by $\Omega(W)$, we can maintain a good competitive ratio and still be robust against a non-oblivious adversary.  This enables us to obtain the following result.  We state the result briefly here and present it formally in Appendix \ref{Appendix:Online_Adv}.  In particular, we defer a detailed explanation about the non-oblivious adversary we consider to the appendix and note here only that it is a \emph{semi-adaptive} adversary.  Informally, the semi-adaptive adversary can design the hitting cost function $f_t$ adaptively based on the algorithm's decision before timestep $t$. However, it must commit its decision point $x_t^*$ as soon as $f_t$ is revealed. 
This adversarial model is the analog of the notion of pseudo-regret for adversarial bandits studied in \cite{bubeck2012regret}.

\begin{algorithm}[t]
\caption{Randomized SFHC (Version B)}\label{alg:RRFHC_v2}
\begin{algorithmic}[1]
\State Generate a sequence of timesteps $R = \{t_0, t_1, \cdots \}$ such that $t_0 = 0, t_{i+1} = t_i + Y_i$. $\{Y_i\}$ are i.i.d uniform randomly picked in $\{n \in \mathbb{Z} \mid \frac{w}{2} < n \leq w - 1\}$. 
\State Pick 
\begin{equation*}
    \begin{aligned}
    &\argmin_x \sum_{\tau=1}^T f_\tau(x_\tau) + c(x_\tau, x_{\tau-1})\\
    &\text{Subject to }x_\tau = v_\tau, \forall \tau \in R.
    \end{aligned}
\end{equation*}
\end{algorithmic}
\end{algorithm}

\begin{theorem}\label{thm:RRFHC_general_cr_1}
Consider an online optimization problem where the movement cost and hitting costs satisfy Conditions I and II.   Randomized SFHC (Version B) has an expected cost that is bounded in terms of the offline optimal as follows: 
$$\mathbb{E}\left(cost(ALG)\right) \leq \left( 1 + \frac{2}{w - 2}\max\left(\frac{\eta}{\lambda}, 2(\eta - 1) \right)\right) \mathbb{E}\left(cost(ADV)\right),$$ 
given a semi-adaptive adversary and $w \geq 4$ predictions.
\end{theorem}

\noindent Note that this theorem requires $w\geq 4$ for technical reasons related to the timing of the randomness injected by the algorithm.
\section{Concluding Remarks}


In this paper we have studied the problem of online optimization with movement costs and, for the first time, provided algorithms with provable guarantees for the case when the hitting costs are non-convex. More specifically, we presented two simple conditions on the hitting and movement costs that are sufficient to guarantee the existence of a competitive algorithm and general enough to capture most previously studied settings.  These conditions do not rely on convexity. We also presented a novel algorithm, Synchronized Fixed Horizon Control (SFHC), and showed that it has a constant, dimension-free competitive ratio even in the non-convex setting. This marks the first time a competitive algorithm for online optimization with movement costs and non-convex costs has appeared in the literature. Further, SFHC is the first algorithm that both provides a constant dimension-free competitive ratio when given no predictions \textit{and} leverages predictions to achieve a competitive ratio that converges to 1 as $w\to\infty$. 

There are several interesting directions for future work motivated by the results here. First, our results focus on the case when predictions are perfect and, in practice, predictions are rarely perfect.  In the case of SOCO, there have recently emerged some results for settings with noisy predictions \cite{chen2015online, chen2016using}, and it will be interesting to understand if it is possible to extend such results to non-convex settings.  Another important direction where there has been success in the convex setting is in the study of distributed algorithms for online optimization, e.g.,  the two-timescale model introduced in \cite{goel2017thinking}.  Again, extending such results to the non-convex setting is an important, and likely challenging, task.  Finally, while we use randomization in SFHC in the non-convex setting, it is not yet clear whether this is necessary.  In particular, in the convex setting it is known that randomization is not needed to achieve the optimal competitive ratio \cite{bansal20152}, but in the non-convex setting no such result exists.  Understanding lower bounds on what is achievable via deterministic and randomized algorithms in the non-convex setting remains open. 
\bibliographystyle{ACM-Reference-Format}
\bibliography{main}

\appendix

\section{Preliminaries}\label{Appendix:Pre}

The appendices that follow provide proofs of the results in the body of the paper.  Throughout the proofs in the appendix we use the following notation to denote the hitting and movement costs of the online learner: $H_t = f_t(x_t)$ and $M_t = c(x_t, x_{t-1})$, where $x_t$ is the point picked by the algorithm at timestep $t$.
Similarly, we denote the hitting and movement costs of the offline optimal (or offline/online adversary) as $H_t^* = f_t(x_t^*)$ and $M_t^* = c(x_t^*, x_{t-1}^*)$, where $x_t^*$ is the point picked by the offline optimal at timestep $t$. For convenience, we use $x_{\tau:t}$ to denote a sequence of points (or random variables) $x_\tau, x_{\tau + 1}, \cdots, x_t$. 




\section{Proof of Proposition \ref{prop:Reduce_SOCO_to_CBC}}

\begin{figure}
\begin{center}
    \includegraphics{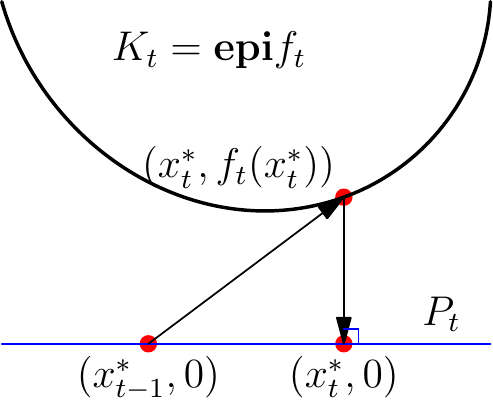}
\end{center}
\caption{\emph{\textbf{An illustration for the proof of Lemma \ref{lemma:ratio_CBC_SOCO_opt}.} Suppose the offline optimal of SOCO picks $x_t^* \in \mathbb{R}^d$ at timestep $t$. We instruct the offline CBC player to pick $(x_t^*, f_t(x_t^*)) \in K_t$ and $(x_t, 0) \in P_t$. Recall that $K_t$ is the epigraph of $f_t$ and $P_t$ is the hyperplane $\{(x, 0)\mid x \in \mathbb{R}^d\}$.}}
\label{figure:CBC_Reduction_OPT}
\end{figure}
\begin{figure}
\begin{center}
    \includegraphics{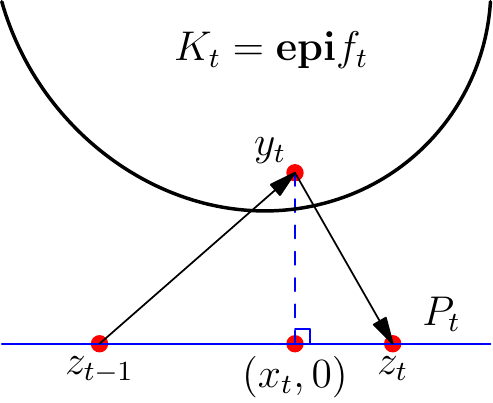}
\end{center}
\caption{\emph{\textbf{An illustration for the proof of Lemma \ref{lemma:ratio_CBC_SOCO_alg}.} Suppose the online CBC Algorithm picks $y_t \in K_t$ at timestep $2t - 1$ and $z_t \in P_t$ at timestep $2t$. In the reduction, we instruct the online SOCO Algorithm to pick $x_t$ which is built by the first $d$ elements of $y_t$. As shown in the figure, $(x_t, 0)$ is the projection of $y_t$ on $P_t$. Recall that $K_t$ is the epigraph of $f_t$ and $P_t$ is the hyperplane $\{(x, 0)\mid x \in \mathbb{R}^d\}$.}}
\label{figure:CBC_Reduction_ALG}
\end{figure}

To begin, consider the SOCO problem.  Suppose the starting point is $x_0 \in \mathbb{R}^d$ and the hitting cost function at timestep $t$ is $f_t$. In the reduction from SOCO to CBC, we convert each hitting cost function $f_t$ to its epigraph $K_t \subseteq \mathbb{R}^{d+1}$ Recall that for a non-negative convex function $f: \mathbb{R}^d\to \mathbb{R}_{\geq 0}$, the epigraph of $f$ is defined to be the convex set
$\mathbf{epi} f = \{(x, y)\mid x \in \mathbf{dom} f, y \geq f(x)\}$. 

We pick the starting point at $z_0 = (x_0, 0)$ and instruct the CBC algorithm $A$ to chase a sequence of $2T$ convex sets and hyperplanes in this order: $K_1, P_1, K_2, P_2, \cdots, K_T, P_T$, where the hyperplane $P_t = \{(x, 0)\mid x \in \mathbb{R}^d\}, \forall t \in [T]$. 
Suppose the point sequence picked by $A$ is $y_1, z_1, y_2, z_2, \cdots, y_t, z_T$, where $y_t \in K_t, z_t \in P_t, \forall t \in [T]$. After $A$ commits its solution $y_t$ in body $K_t$, our SOCO algorithm $A'$ picks the first $d$ elements in $y_t$ as the solution $x_t$ of SOCO at timestep $t$.


For convenience, we use the $OPT_{CBC}$ to denote the offline optimal cost of chasing the body sequence $K_1, P, K_2, P, \cdots, K_T, P$, i.e.
\begin{equation}\label{thm:Reduce_SOCO_to_CBC:e1}
    \begin{aligned}
    OPT_{CBC} =& \min_{y_{1:T}^*, z_{1:T}^*}\sum_{t=1}^T \left(\norm{y_t^* - z_{t-1}^*} + \norm{z_t^* - y_t^*}\right)\\
    &\text{subject to }y_t^* \in K_t, z_t^* \in P_t, \forall t \in [T], z_0^* = z_0.
    \end{aligned}
\end{equation}
In the same problem setting, we use $ALG_{CBC}$ to denote the total cost incurred by Algorithm $A$, i.e.
\begin{equation}\label{thm:Reduce_SOCO_to_CBC:e2}
    ALG_{CBC} = \sum_{t=1}^T \left(\norm{y_t - z_{t-1}} + \norm{z_t - y_t}\right).
\end{equation}

We define $OPT_{SOCO}$ to be the offline optimal cost of SOCO problem where the hitting cost function sequence is given by $f_1, f_2, \cdots, f_t$, i.e.
\begin{equation}\label{thm:Reduce_SOCO_to_CBC:e3}
    OPT_{SOCO} = \min_{x_{1:T}^*}\sum_{t=1}^T \left(f_t(x_t^*) + \norm{x_t^* - x_{t-1}^*}\right).
\end{equation}
In the same problem setting, we use $ALG_{SOCO}$ to denote the total cost incurred by Algorithm $A'$, i.e.
\begin{equation}\label{thm:Reduce_SOCO_to_CBC:e4}
    ALG_{SOCO} = \sum_{t=1}^T \left(f_t(x_t) + \norm{x_t - x_{t-1}}\right).
\end{equation}

The proof follows immediately from two lemmas. We state the lemmas and conclude the proof of Proposition \ref{prop:Reduce_SOCO_to_CBC} before proving the lemmas.

\begin{lemma}\label{lemma:ratio_CBC_SOCO_opt}
In the reduction, the optimal cost of the CBC problem is less than or equal to $2$ times the optimal cost of the SOCO problem, i.e.
$$OPT_{CBC} \leq 2\cdot OPT_{SOCO},$$
where $OPT_{CBC}, OPT_{SOCO}$ are defined in equations \eqref{thm:Reduce_SOCO_to_CBC:e1} and \eqref{thm:Reduce_SOCO_to_CBC:e3}.
\end{lemma}

\begin{lemma}\label{lemma:ratio_CBC_SOCO_alg}
In the reduction, the cost incurred by SOCO algorithm $A'$ is less than or equal to $2$ times the cost incurred by CBC algorithm $A$, i.e.
$$ALG_{SOCO} \leq 2\cdot ALG_{CBC},$$
where $ALG_{CBC}, ALG_{SOCO}$ are defined in equations \eqref{thm:Reduce_SOCO_to_CBC:e2} and \eqref{thm:Reduce_SOCO_to_CBC:e4}.
\end{lemma}

To complete the proof the proposition using these lemmas, we note that Lemma \ref{lemma:ratio_CBC_SOCO_opt} and Lemma \ref{lemma:ratio_CBC_SOCO_alg} give us that the sequence of hitting cost functions $f_1, \cdots, f_T$ satisfies
\begin{equation*}
    \begin{aligned}
    \frac{ALG_{SOCO}}{OPT_{SOCO}} &= \frac{ALG_{SOCO}}{ALG_{CBC}}\cdot \frac{ALG_{CBC}}{OPT_{CBC}}\cdot \frac{OPT_{CBC}}{OPT_{SOCO}}\\
    &\leq 2 \cdot C \cdot 2\\
    &= 4C.
    \end{aligned}
\end{equation*}

All that remains is to prove Lemma \ref{lemma:ratio_CBC_SOCO_opt} and \ref{lemma:ratio_CBC_SOCO_alg}.

\begin{proof}[Proof of Lemma \ref{lemma:ratio_CBC_SOCO_opt}]
Suppose the point sequence $x_1^*, x_2^*, \cdots, x_T^*$ is picked by the offline optimal in the SOCO problem. Let $y_t' = (x_t^*, f_t(x_t^*)), z_t' = (x_t^*, 0)$.  See Figure \ref{figure:CBC_Reduction_OPT} for an illustration. 

Thus, we obtain that
\begin{subequations}\label{lemma:ratio_CBC_SOCO_opt:e1}
\begin{align}
OPT_{CBC} &\leq \sum_{t=1}^T\left(\norm{y_t' - z_{t-1}'}_p + \norm{z_t' - y_t'}_p\right)\nonumber\\
&= \sum_{t=1}^T \left(\norm{(x_t^* - x_{t-1}^*, f_t(x_t^*))}_p + \norm{(\bm{0}, f_t(x_t^*))}_p\right)\label{lemma:ratio_CBC_SOCO_opt:e1_0}\\
&\leq \sum_{t=1}^T \left(\norm{(x_t^* - x_{t-1}^*, 0)}_p + 2\norm{(\bm{0}, f_t(x_t^*))}_p\right)\label{lemma:ratio_CBC_SOCO_opt:e1_1}\\
&= \sum_{t=1}^T \left(\norm{x_t^* - x_{t-1}^*}_p + 2f_t(x_t^*)\right)\label{lemma:ratio_CBC_SOCO_opt:e1_2}\\
&\leq 2\sum_{t=1}^T \left(\norm{x_t^* - x_{t-1}^*}_p + f_t(x_t^*)\right)\nonumber\\
&= 2\cdot OPT_{SOCO},\nonumber
\end{align}
\end{subequations}
where we use $\bm{0}$ to denote $d$ zeros in \eqref{lemma:ratio_CBC_SOCO_opt:e1_0}; the triangle inequality in \eqref{lemma:ratio_CBC_SOCO_opt:e1_1}; the definition of $\ell_p$ norm and the fact that $f_t$ is non-negative in \eqref{lemma:ratio_CBC_SOCO_opt:e1_2}.
\end{proof}

\begin{proof}[Proof of Lemma \ref{lemma:ratio_CBC_SOCO_alg}]
Recall that we assume $y_1, z_1, \cdots, y_T, z_T$ is the point sequence picked by $A$ in the CBC problem, and $x_t = (y_{t, 1}, y_{t, 2}, \cdots, y_{t, d})$,  i.e., the first $d$ elements of $y_t$. Notice that $(x_t, 0)$ is the projection of $y_t$ onto the hyperplane $P_t$.  See Figure \ref{figure:CBC_Reduction_ALG}) for an illustration.

It follows that, 
\begin{subequations}\label{lemma:ratio_CBC_SOCO_alg:e1}
\begin{align}
&ALG_{SOCO}\nonumber\\
={}& \sum_{t=1}^T\left(f_t(x_t) + \norm{x_t - x_{t-1}}_p\right)\nonumber\\
\leq{}& \sum_{t=1}^T\left(y_{t, d+1} + \norm{(y_{t, 1} - y_{t-1, 1}, \cdots, y_{t, d} - y_{t-1, d})}_p\right)\label{lemma:ratio_CBC_SOCO_alg:e1_0}\\
\leq{}& \sum_{t=1}^T\Big(y_{t, d+1} + \norm{(z_{t, 1} - y_{t-1, 1}, \cdots, z_{t, d} - y_{t-1, d})}_p\nonumber\\
&+ \norm{(y_{t, 1} - z_{t, 1}, \cdots, y_{t, d} - z_{t, d})}_p\Big)\label{lemma:ratio_CBC_SOCO_alg:e1_1}\\
\leq{}& \sum_{t=1}^T\left( \norm{z_t - y_t} + \norm{z_t - y_{t-1}} + \norm{y_t - z_t}\right)\label{lemma:ratio_CBC_SOCO_alg:e1_2}\\
\leq{}& 2\sum_{t=1}^T\left( \norm{z_t - y_t} + \norm{z_t - y_{t-1}}\right)\nonumber\\
={}& 2\cdot ALG_{CBC},\nonumber
\end{align}
\end{subequations}
where we use the fact that $y_t \in \mathbf{epi} f_t$ in \eqref{lemma:ratio_CBC_SOCO_alg:e1_0}; the triangle inequality in \eqref{lemma:ratio_CBC_SOCO_alg:e1_1}; and the definition of $\ell_p$ norm in \eqref{lemma:ratio_CBC_SOCO_alg:e1_2}.
\end{proof}

\section{A Connection to \cite{lin2012online}} \label{Appendix:Geo}

In this section we show that the problem setting introduced in \cite{lin2012online} can be viewed as a special case of the sufficient conditions introduced in Section \ref{sec:pred_helps}.  Note that we make some small modifications to notation in order to avoid conflicts with other parts of this paper.

In \cite{lin2012online}, the feasible set $\mathcal{F}$ is defined as $\mathcal{F} = \{x \in \mathbb{R}^d \mid x_i \geq 0, \forall i\}$. The $d-$dimensional vector $x_t = (x_{t,1}, x_{t,2}, \cdots, x_{t, d}) \in \mathcal{F}$ represents the number of different kinds of servers running at timestep $t$. The hitting cost function $f_t$ is assumed to satisfy
\begin{equation}\label{GLB_e1}
    f_t(x) \geq e_0 \cdot x, \forall x \in \mathcal{F},
\end{equation}
where $e_0 \in \left(\mathbb{R}^+\right)^d$ is the minimum cost to operate each kind of the servers. The switching/movement cost is given by
\begin{equation}\label{GLB_e2}
    c(x_t, x_{t-1}) = \beta \cdot (x_t - x_{t-1})^+,
\end{equation}
where $\beta \in \left(\mathbb{R}^+\right)^d$ is the cost to start each kind each kind of the servers.

To connect this model to the sufficient conditions in Section \ref{sec:pred_helps} note that, since $v_t$ is the global minimum of $f_t$, we have
\begin{equation}\label{GLB_e3}
    f_t(x) \geq f_t(v_t) \geq e_0\cdot v_t,
\end{equation}
where we use condition \eqref{GLB_e1} in the last step.

Therefore, we can lower bound $f_t(x)$ by
\begin{subequations}\label{GLB:e4}
\begin{align}
    f_t(x)&\geq \frac{1}{2}\left(e_0 \cdot (x + v_t)\right)\label{GLB:e4:s1}\\
    &= \frac{1}{2}\sum_{s=1}^d e_{0,s}(x_s + v_{t,s})\nonumber\\
    &\geq \frac{1}{2}\sum_{s=1}^d e_{0,s}\left((x_s - v_{t,s})^+ + (v_{t,s} - x_s)^+\right)\label{GLB:e4:s2}\\
    &= \frac{1}{2}\sum_{s=1}^d \frac{e_{0,s}}{\beta_s}\left(\beta_s(x_s - v_{t,s})^+ + \beta_s(v_{t,s} - x_s)^+\right)\nonumber\\
    &\geq \frac{1}{2}\min_s \frac{e_{0,s}}{\beta_s}\cdot \left(c(x, v_t) + c(v_t, x)\right),\label{GLB:e4:s3}
\end{align}
\end{subequations}
where we combine \eqref{GLB_e3} with \eqref{GLB_e1} in \eqref{GLB:e4:s1}; the fact that $x + y \geq (x - y)^+ + (y - x)^+, \forall x, y \in \mathbb{R}^+\cup\{0\}$ in \eqref{GLB:e4:s2}; \eqref{GLB_e2} in \eqref{GLB:e4:s3}.

\section{Proof of Theorem \ref{thm:AFHC_non_convex_1}}\label{Appendix:AFHC_non_convex_1}
Lemma \ref{lemma:ARFHC_general_cr} gives us that
\begin{equation}\label{Coro_AFHC_non_convex:e1}
    \sum_{h=1}^w cost\left(SFHC(h)\right) \leq \left(1 + \frac{1}{w}\max\left(\frac{\eta}{\lambda}, 2(\eta - 1)\right)\right)\cdot cost(OPT).
\end{equation}
By Lemma \ref{lem:AFHC_non_convex_1}, we also have that
\begin{equation}\label{Coro_AFHC_non_convex:e2}
    cost(DSFHC) \leq \frac{1}{w}\left(1 + \frac{\alpha}{\lambda}\right)\sum_{h=1}^w cost\left(SFHC(h)\right).
\end{equation}
Combining \eqref{Coro_AFHC_non_convex:e1} and \eqref{Coro_AFHC_non_convex:e2}, we obtain the conclusion of the theorem.

\section{A Semi-Adaptive Adversary}\label{Appendix:Online_Adv}
Our main result in Section \ref{subsec:rand} focuses on the case of an oblivious adversary, but it is interesting from an algorithmic perspective to understand if it is possible to obtain results that hold in the context of non-oblivious adversaries.  We are indeed able to obtain such a result (Theorem \ref{thm:RRFHC_general_cr_1}) for a form of semi-adaptive adversary that was first introduced in the context of adversarial bandits (see Chapter 3 of \cite{bubeck2012regret}). We introduce the semi-adaptive adversarial model here and then prove Theorem \ref{thm:RRFHC_general_cr_1}.  

Suppose for each $\tau \in [T]$ we have a function $\Psi_\tau: \mathbb{R}^d \to \mathcal{D}_\tau$, where $\mathcal{D}_\tau$ contains a finite number of elements. These functions are determined before the game starts and are common knowledge between the online adversary and the learner. Then, a game is played between the adversary and the learner as follows.

\begin{enumerate}
    \item \textbf{At time $\tau=0$:} The online adversary designs hitting costs $f_1, f_2, \cdots, f_{w - 1}$ and reveals them to the learner. The online adversary then commits its choices for these timesteps, $x_1^*, x_2^*, \cdots, x_{w - 1}^*$, but these points are not revealed to the learner.
    \item \textbf{At time $\tau \geq 1$:} If $\tau + w - 1 \leq T$, the online adversary designs hitting cost $f_{\tau+w-1}$ and reveals it to the learner. The online adversary must also commit its choice for timestep $\tau+w-1$, $x_{\tau+w-1}^*$, but this choice is not revealed to the learner. The learner then observes $f_{\tau+w-1}$ and selects $x_{\tau}$. The information $z_{\tau} = \Psi_{\tau}(x_{\tau})$ is revealed to the online adversary.
\end{enumerate}
Intuitively, this means that the adversary may adapt the choice of future cost functions to the choices of the learner, but cannot change a cost function once it is revealed as part of the prediction window of the learner.  Further, the adversary must commit to its actions when revealing the cost function, and cannot wait until seeing the decision of the learner.  Thus, the adversary is not completely adaptive, it cannot wait to choose its action until the end of the instance, but it does have the ability to partially adapt to the decisions of the learner.   

In this game setting, the online adversary's cost is
$$cost(ADV) = \sum_{\tau = 1}^T f_\tau(x_\tau^*) + c(x_\tau^*, x_{\tau-1}^*)$$
and our objective is to prove a result in the form of
$$\mathbb{E} cost(ALG) \leq \left(1 + O\left(\frac{1}{w}\right)\right)\mathbb{E} cost(ADV).$$

Notice that when the algorithm is randomized, the online adversary is a random variable. For any online adversary, we must have $cost(OPT) \leq cost(ADV)$ because, when compared with the online adversary, the offline optimal can change its decision on $x_\tau^*$ after timestep $\tau$. However, the online adversary can be more powerful than any oblivious adversary. When the algorithm is deterministic, the best semi-adaptive adversary is identical with the best oblivious online adversary. On the other hand, when the algorithm is randomized, the semi-adaptive adversary will generate hitting costs based on the previous behaviour of the algorithm instead of fixing the hitting costs sequence and the point sequence at the beginning of the game. 

To understand why a semi-adaptive adversary requires changes to Randomized SFHC, consider the following intuition. If the algorithm chooses to synchronize at timestep $\tau$ (pick $x_\tau = v_\tau$), we call $\tau$ an \textit{anchor timestep}. For example, in Randomized SFHC (Version B) (see Algorithm \ref{alg:RRFHC_v2}), $R$ is the set of anchor timesteps. The extra cost incurred by the algorithm can be upper bounded a constant times the offline optimal cost at the same timestep (see Lemma \ref{thm:RRFHC_general_cr_1:lemma_1}). Due to the length limitation of the prediction window, the distance between two anchor timesteps cannot be larger than $w$. Before revealing $f_\tau$, if the online adversary guesses that $\tau$ is an anchor timestep with high probability, it can manipulate the hitting cost function $f_\tau$ so that the offline optimal cost at timestep $\tau$ is much larger than the sum of all other timesteps. If the guess turns out to be correct, the performance bound of our algorithm is extremely bad. To avoid this possibility, we use randomization to pick the anchor timesteps; thus limiting the chance that the adversary guesses correctly.  In particular, let $E_\tau$  denote the event that $\tau$ is picked as an anchor timestep. we can reduce the conditional probability $P\left(E_\tau \mid z_{1:\tau-w}\right)$ to be $O\left(\frac{1}{w}\right)$ for all possible $z_{1:\tau-w}$, where $z_{1:\tau-w}$ is all the information revealed by the learner before $f_\tau$ is decided. Therefore, the online adversary cannot predict $E_\tau$ better than random guess by a constant factor.  This intuition guides the proof that follows.

\begin{proof}[Proof of Theorem \ref{thm:RRFHC_general_cr_1}]

Let $Z_\tau = \Psi_\tau(x_\tau)$ denote the information revealed by the algorithm at time $\tau - 1$. Therefore, $Z_\tau$ is a random variable over the set $\mathcal{D}_\tau$. For convenience, we use $Z_{\tau_1:\tau_2}$ to denote all the information revealed after the adversary takes action at time $\tau_1$ and before it takes action at time $\tau_2+1$. We use $z_{\tau_1:\tau_2}$ to denote a possible outcome of $Z_{\tau_1:\tau_2}$. Therefore, we can assume $f_{\tau+w}$, $H_{\tau+w}^*$ and $M_{\tau+w}^*$ are functions on input $z_{1:\tau}$.

To begin, we state a lemma showing that no matter what set $\Omega$ of synchronized timesteps is chosen, the total cost of Randomized SFHC (Version B) can be bounded by the total cost of the online adversary plus the online adversary costs incurred at timesteps in $\Omega$.  The proof of the lemma follows the completion of the proof of Theorem \ref{thm:RRFHC_general_cr_1}. 

\begin{lemma}\label{thm:RRFHC_general_cr_1:lemma_1}
Under the same assumptions of Theorem \ref{thm:RRFHC_general_cr_1}, suppose a timestep sequence $\Omega = \{t_0, t_1, \cdots\}$ satisfies $t_0 = 0, t_i \geq t_{i-1} + 2$. For any sequence $x_{1:T}'$, the total cost incurred by the algorithm  which picks the minimizers of the optimization problem \eqref{thm:RRFHC_general_cr_1:lemma_1:opt}
\begin{equation}\label{thm:RRFHC_general_cr_1:lemma_1:opt}
    \begin{aligned}
    &\min_x \sum_{\tau=1}^T f_\tau(x_\tau) + c(x_\tau, x_{\tau-1})\\
    &\text{subject to }x_\tau = v_\tau, \forall \tau \in \Omega\cap [T].
    \end{aligned}
\end{equation}
is upper bounded by
$$\sum_{\tau=1}^T (H_\tau' + M_\tau') + \frac{\eta}{\lambda}\sum_{s \in \Omega \cap [H]}H_s' + (\eta - 1)\sum_{s \in \Omega \cap [H]}\left(M_s' + M_{s+1}'\right),$$
where we use the notation $H_\tau' = f_t(x_\tau'), M_\tau' = c(x_\tau', x_{\tau-1}')$.
\end{lemma}

Lemma \ref{thm:RRFHC_general_cr_1:lemma_1} can be interpreted as follows: Randomized SFHC incurs an extra cost of $ \frac{\eta}{\lambda}H_s^* + (\eta - 1)(M_s^* + M_{s+1}^*)$ if $s \in R$. For convenience, we define a random variable
$$extra(s) = \begin{cases}
\frac{\eta}{\lambda}H_s^* + (\eta - 1)(M_s^* + M_{s+1}^*) & s \in R\\
0 & \text{ Otherwise}.
\end{cases}$$
Next we state a one more technical lemma crucial to our proof.  Again we defer its proof until after the proof of Theorem \ref{thm:RRFHC_general_cr_1}. 

\begin{lemma}\label{thm:RRFHC_general_cr_1:lemma_2}
Under the same assumptions of Theorem \ref{thm:RRFHC_general_cr_1}, if we use the first step of Randomized RFHC (Algorithm \ref{alg:RRFHC_v2}) to generate the timestep sequence $R$, we have
$$\forall \tau \in \{1, 2, \cdots, T\}, P(\tau \in R\mid Z_{1:\tau-w+1} = z_{1:\tau-w+1}) \leq \frac{2}{w - 2}.$$
\end{lemma}

Recalling that $H_\tau^*$ and $M_\tau^*$ are functions of previous information $z_{1:\tau-w}$, we can apply Lemma \ref{thm:RRFHC_general_cr_1:lemma_1} and Lemma \ref{thm:RRFHC_general_cr_1:lemma_2} to conclude that
\begin{subequations}\label{thm:RRFHC_general_cr_1:e1}
\begin{align}
    &\mathbb{E} cost(RSFHC) \nonumber\\
    \leq{}& \sum_{\tau=1}^T \sum_{z_{1:\tau}}P(Z_{1:\tau} = z_{1:\tau})\left(\left(H_\tau^* + M_\tau^*\right) + \mathbb{E}[extra(\tau)\mid Z_{1:\tau} = z_{1:\tau}]\right)\label{thm:RRFHC_general_cr_1:e1:s1}\\
    \leq{}& \mathbb{E} cost(ADV) + \sum_{\tau=1}^T\sum_{z_{1:\tau}}P(Z_{1:\tau} = z_{1:\tau})\mathbb{E}[extra(\tau)\mid Z_{1:\tau} = z_{1:\tau}].\nonumber
\end{align}
\end{subequations}
In \eqref{thm:RRFHC_general_cr_1:e1:s1}, notice that $H_\tau^*, M_\tau^*$ and $extra(\tau)$ are all determined when $Z_{1:\tau} = z_{1:\tau}$ is given.

We also see that
\begin{subequations}\label{thm:RRFHC_general_cr_1:e2}
\begin{align}
    &\sum_{\tau=1}^T\sum_{z_{1:\tau}}P(Z_{1:\tau} = z_{1:\tau})\mathbb{E}[extra(\tau)\mid Z_{1:\tau} = z_{1:\tau}]\nonumber\\
    ={}& \sum_{\tau=1}^T\sum_{z_{1:\tau-w+1}}P(Z_{1:\tau-w+1} = z_{1:\tau-w+1})\nonumber\\
    &\cdot P(\tau \in R \mid Z_{1:\tau-w+1} = z_{1:\tau-w+1})\nonumber\\
    &\cdot \left( \frac{\eta}{\lambda}H_\tau^* + (\eta - 1)(M_\tau^* + M_{\tau+1}^*) \right)\label{thm:RRFHC_general_cr_1:e2:s1}\\
    \leq{}& \frac{2}{w - 2}\sum_{\tau=1}^T\sum_{z_{1:\tau-w+1}}P(Z_{1:\tau-w+1} = z_{1:\tau-w+1})\nonumber\\
    &\cdot \left( \frac{\eta}{\lambda}H_\tau^* + (\eta - 1)(M_\tau^* + M_{\tau+1}^*) \right)\label{thm:RRFHC_general_cr_1:e2:s2}\\
    \leq{}& \frac{2}{w - 2}\max\{\frac{\eta}{\lambda}, 2(\eta - 1)\}\mathbb{E} cost(ADV).\label{thm:RRFHC_general_cr_1:e2:s3}
\end{align}
\end{subequations}
In \eqref{thm:RRFHC_general_cr_1:e2:s1}, notice that for a fixed sequence of outcomes $z_{1:\tau-w+1}$, we have
\begin{equation*}
    \begin{aligned}
    &\sum_{z_{\tau-w+2:\tau}}P(Z_{1:\tau} = z_{1:\tau})\mathbb{E}[extra(\tau)\mid Z_{1:\tau} = z_{1:\tau}]\\
    ={}& \sum_{z_{\tau-w+2:\tau}}P(\tau \in R \text{ and }Z_{1:\tau} = z_{1:\tau})\\
    &\cdot \left( \frac{\eta}{\lambda}H_\tau^* + (\eta - 1)(M_\tau^* + M_{\tau+1}^*) \right)\\
    ={}& P(\tau\in R \text{ and } Z_{1:\tau-w+1} = z_{1:\tau-w+1})\\
    &\cdot \left( \frac{\eta}{\lambda}H_\tau^* + (\eta - 1)(M_\tau^* + M_{\tau+1}^*) \right)\\
    ={}& P(\tau\in R \mid Z_{1:\tau-w+1} = z_{1:\tau-w+1})\cdot P(Z_{1:\tau-w+1} = z_{1:\tau-w+1})\\
    &\cdot \left( \frac{\eta}{\lambda}H_t^* + (\eta - 1)(M_\tau^* + M_{\tau+1}^*) \right).
    \end{aligned}
\end{equation*}
We use Lemma \ref{thm:RRFHC_general_cr_1:lemma_2} to bound the conditioned probability in \eqref{thm:RRFHC_general_cr_1:e2:s2}. Recall that $H_\tau^*$ and $M_\tau^*$ are given by a function of $z_{1:\tau-w}$, therefore, \eqref{thm:RRFHC_general_cr_1:e2:s3} follows from the fact that
$$\mathbb{E} cost(ADV) = \sum_{\tau=1}^T \sum_{z_{1:\tau-w}} P(Z_{1:\tau-w} = z_{1:\tau-w}) (H_\tau^* + M_\tau^*).$$

Substituting \eqref{thm:RRFHC_general_cr_1:e2} into \eqref{thm:RRFHC_general_cr_1:e1}, we obtain that
$$\mathbb{E} cost(RSFHC) \leq \left(1 + \frac{2}{w - 2}\max\{\frac{\eta}{\lambda}, 2(\eta - 1)\}\right)\mathbb{E} cost(ADV).$$
\end{proof}

We end the section by proving Lemmas \ref{thm:RRFHC_general_cr_1:lemma_1} and \ref{thm:RRFHC_general_cr_1:lemma_2}.

\begin{proof}[Proof of Lemma \ref{thm:RRFHC_general_cr_1:lemma_1}]
Suppose $\Omega \bigcap [T] = \{t_0, t_1, \cdots, t_k\}$.

We adopt the definition of $g_{\tau_1, \tau_2}$ in \eqref{equ:g_function_1} and \eqref{equ:g_function_2} in Section \ref{sec:pred_helps}. Using this notation, the algorithm we consider selects the minimizers of $g_{t_i, t_{i+1}}$ as its choice $x_{t_i+1:t_{i+1}-1}$ for $i \leq k-1$. The algorithm also selects the minimizers of $g_{t_k, t_{k+1}}$ as its choice $x_{t_k+1:T}$. For notation convenience, we define $x_{t_k+1:t_{k+1}-1} := x_{t_k+1:T}$.

Therefore, for all $0 \leq i \leq k$, since $x_{t_i+1: t_{i+1}-1}$ is the minimizer of $g_{t_i, t_{i+1}}$, we have that
\begin{equation}\label{lemma:RRFHC_core_1:e1}
    g_{t_i + 1, t_{i+1} - 1}(x_{t_i + 1: t_{i+1} - 1}) \leq g_{t_i + 1, t_{i+1} - 1}(x_{t_i + 1: t_{i+1} - 1}'),
\end{equation}
where $x$ is the solution picked by the algorithm and $x'$ is the sequence we want to compare with.

Notice that $\forall 0 \leq s \leq T - 1$, we have
\begin{equation}\label{lemma:RRFHC_core_1:e2_1}
    c(x_{s+1}', v_{s}) \leq \eta \left(c(x_{s+1}', x_{s}') + c(x_{s}', v_{s})\right).
\end{equation}
by the Approximate Triangle Inequality (Condition II), and 
\begin{equation}\label{lemma:RRFHC_core_1:e2_2}
    \begin{aligned}
    c(v_{s + 1}, x_{s}') &\leq \eta \left(c(x_{s + 1}', x_{s}') + c(v_{s + 1}, x_{s + 1}')\right)\\
    &\leq \eta \left(c(x_{s + 1}', x_{s}') + \frac{1}{\lambda}f_{s + 1}(x_{s + 1}') - c(x_{s + 1}', v_{s + 1})\right),
    \end{aligned}
\end{equation}
by both Condition I and II.

Summing \eqref{lemma:RRFHC_core_1:e1} over $0 \leq i \leq k$, we obtain that
\begin{subequations}\label{lemma:RRFHC_core_1:e2}
\begin{align}
    cost(x) ={}& \sum_{i = 0}^k g_{t_i + 1, t_{i+1} - 1}(x_{t_i + 1: t_{i+1} - 1})\nonumber\\
    \leq{}& \sum_{i = 0}^k g_{t_i + 1, t_{i+1} - 1}(x_{t_i + 1: t_{i+1} - 1}')\label{lemma:RRFHC_core_1:e2:s1}\\
    ={}& \sum_{\tau=1}^T \left(H_\tau' + M_\tau'\right) - \sum_{i=1}^k H_{t_i}' + \sum_{i=1}^k f_{t_i}(v_{t_i})\nonumber\\
    &+ \sum_{i=1}^k \left(c(x_{t_i+1}', v_{t_i}) - c(x_{t_i+1}', x_{t_i}')\right)\nonumber\\
    &+ \sum_{i=1}^k \left(c(v_{t_i}, x_{t_i - 1}') - c(x_{t_i}', x_{t_i - 1}')\right)\nonumber\\
    \leq{}& \sum_{\tau=1}^T \left(H_\tau' + M_\tau'\right) + \frac{\eta}{\lambda}\sum_{s \in \Omega \cap [T]}H_s'\nonumber\\
    &+ (\eta - 1)\sum_{s \in \Omega \cap [T]}\left(M_s' + M_{s+1}'\right),\label{lemma:RRFHC_core_1:e2:s2}
\end{align}
\end{subequations}
where we use \eqref{lemma:RRFHC_core_1:e1} in \eqref{lemma:RRFHC_core_1:e2:s1}; \eqref{lemma:RRFHC_core_1:e2_1} and \eqref{lemma:RRFHC_core_1:e2_2} in \eqref{lemma:RRFHC_core_1:e2:s2}.

\end{proof}

\begin{proof}[Proof of Lemma \ref{thm:RRFHC_general_cr_1:lemma_2}]
There must exists one and exactly one element of set $R$ in the time interval $[\tau - w + 1, \tau - 1]$. We denote this element by $t_q$. Notice that:

If $t_q \in [\tau - w + 1, \tau - \frac{w}{2}]$, we must have $t_{q+2} > t$. Therefore, we see that
$$P(\tau \in R \mid t_q) = P(\tau = t_{q+1} \mid t_q) = P(Y_q = \tau - t_q) \leq \frac{2}{w - 2}.$$

Else, we must have $t_q \in (\tau - \frac{w}{2}, \tau - 1]$. In this case, we must have $t_{q+1} > \tau$. Therefore, we see that
$$P(\tau \in R \mid t_q) = 0.$$

Therefore, we obtain that
\begin{equation*}
    \begin{aligned}
    &P(\tau \in R\mid Z_{1:\tau-w+1} = z_{1:\tau-w+1})\\
    ={}& \sum_{j = \tau - w + 1}^{\tau - 1}P(\tau \in R\mid t_q = j, Z_{1:\tau-w+1} = z_{1:\tau-w+1})\\ 
    &\cdot P(t_q = j \mid Z_{1:\tau-w+1} = z_{1:\tau-w+1})\\
    ={}& \sum_{j = \tau - w + 1}^{\tau - 1}P(\tau \in R\mid t_q = j) P(t_q = j \mid Z_{1:\tau-w+1} = z_{1:\tau-w+1})\\
    \leq{}& \sum_{j = \tau - w + 1}^{\tau - 1}\frac{2}{w - 2} P(t_q = j \mid Z_{1:\tau-w+1} = z_{1:\tau-w+1})\\
    ={}& \frac{2}{w - 2}.
    \end{aligned}
\end{equation*}
\end{proof}

\end{document}